\documentclass[letterpaper]{article} % DO NOT CHANGE THIS
\usepackage{aaai24}  % DO NOT CHANGE THIS
\usepackage{times}  % DO NOT CHANGE THIS
\usepackage{helvet}  % DO NOT CHANGE THIS
\usepackage{courier}  % DO NOT CHANGE THIS
\usepackage[hyphens]{url}  % DO NOT CHANGE THIS
\usepackage{graphicx} % DO NOT CHANGE THIS
\usepackage{natbib}  % DO NOT CHANGE THIS AND DO NOT ADD ANY OPTIONS TO IT
\usepackage{caption} % DO NOT CHANGE THIS AND DO NOT ADD ANY OPTIONS TO IT
\usepackage{algorithm}
\usepackage{algorithmic}

\usepackage[utf8]{inputenc} % allow utf-8 input
\usepackage[T1]{fontenc}    % use 8-bit T1 
\usepackage{url}            % simple URL typesetting
\usepackage{booktabs}       % professional-quality tables
\usepackage{amsfonts}       % blackboard math symbols
\usepackage{nicefrac}       % compact symbols for 1/2, etc.
\usepackage{microtype}      % microtypography
\usepackage{xcolor}         % colors

\usepackage{times}
\usepackage{soul}
\usepackage{url}
\usepackage[utf8]{inputenc}
\usepackage{graphicx}
\usepackage{amsmath}
\usepackage{amsthm}
\usepackage{booktabs}
\usepackage{algorithm}
\usepackage{algorithmic}
\usepackage{multirow}
\usepackage{adjustbox}
\usepackage[title]{appendix}
\usepackage{appendix}

\usepackage{tikz}
\newcommand*\circled[1]{\tikz[baseline=(char.base)]{
            \node[shape=circle,fill,inner sep=1pt] (char) {\footnotesize \textcolor{white}{#1}};}}

\usepackage{algorithm}
\usepackage{algorithmic}
\usepackage{algorithmic}
\usepackage{algorithm}
\usepackage[algo2e,boxed,ruled]{algorithm2e} 
\SetKwRepeat{Do}{do}{while}
\usepackage{amsthm} 
\usepackage{mathtools}
\usepackage{booktabs} 
 
\DeclareMathOperator*{\argmin}{argmin} 
\usepackage{longtable}
\usepackage[format=plain,indention=0cm, font=small, labelfont=bf]{caption}
\usepackage{amsmath}
\usepackage{amssymb}
\usepackage{subfloat}

\usepackage[utf8]{inputenc} % allow utf-8 input
\usepackage[T1]{fontenc}    % use 8-bit T1 fonts    
\usepackage{url}            % simple URL typesetting
\usepackage{booktabs}       % professional-quality tables
\usepackage{amsfonts}       % blackboard math symbols
\usepackage{nicefrac}       % compact symbols for 1/2, etc.
\usepackage{microtype}      % microtypography
\usepackage{xcolor}         % colors
\usepackage{microtype}
\usepackage{subfig}
\usepackage{enumitem}

\usepackage{diagbox}

\newtheorem{theorem}{Theorem}
\theoremstyle{definition}

\newtheorem{lemma}[theorem]{Lemma}

\newtheorem{assumption}{Assumption}
\newtheorem{corollary}{Corollary}

\newtheorem{innercustomthm}{Theorem}
\newenvironment{customthm}[1]
  {\renewcommand\theinnercustomthm{#1}\innercustomthm}
  {\endinnercustomthm}
\newtheorem{innercustomlemma}{Lemma}
\newenvironment{customlemma}[1]
  {\renewcommand\theinnercustomlemma{#1}\innercustomlemma}
  {\endinnercustomlemma}

\usepackage{newfloat}
\usepackage{listings}
\DeclareCaptionStyle{ruled}{labelfont=normalfont,labelsep=colon,strut=off} % DO NOT CHANGE THIS
\floatstyle{ruled}
\newfloat{listing}{tb}{lst}{}
\floatname{listing}{Listing}
\title{Incorporating Domain Differential Equations into Graph Convolutional Networks to Lower Generalization Discrepancy}
\author{
    \textsuperscript{\rm 1}Yue Sun,
    \textsuperscript{\rm 1}Chao Chen,
    \textsuperscript{\rm 1}Yuesheng Xu,
    \textsuperscript{\rm 2}Sihong Xie,
    \textsuperscript{\rm 1}Rick S. Blum, 
    \textsuperscript{\rm 1}Parv Venkitasubramaniam
}
\affiliations{
    \textsuperscript{\rm 1}Lehigh University, 
    \textsuperscript{\rm 2}Hong Kong University of Science and Technology (Guangzhou) \\
    \{yus516, rb0f, pav309\}@lehigh.edu, \{cha01nbox, xuyuesheng324, xiesihong1\}@gmail.com
}

\usepackage{bibentry}
\begin{document}
\maketitle

\begin{abstract}
Ensuring both accuracy and robustness in time series prediction is critical to many
applications, ranging from urban planning to pandemic management.
With sufficient training data where all spatiotemporal patterns are well-represented, existing deep-learning models
can make reasonably accurate predictions.
However, existing methods fail when the training data are drawn from different circumstances (e.g., traffic patterns on regular days) compared to test data (e.g., traffic patterns after a natural disaster).
Such challenges are usually classified under domain generalization. In this work, we show that one way to address this challenge in the context of spatiotemporal prediction is by incorporating domain differential equations into Graph Convolutional Networks (GCNs).
We theoretically derive conditions where GCNs incorporating such domain differential equations are robust to mismatched training and testing data compared to baseline domain agnostic models.
To support our theory, we propose two domain-differential-equation-informed networks called 
\underline{R}eaction-\underline{D}iffusion \underline{G}raph \underline{C}onvolutional \underline{N}etwork (RDGCN),
which incorporates differential equations for traffic speed evolution, and \underline{S}usceptible-\underline{I}nfectious-\underline{R}ecovered \underline{G}raph \underline{C}onvolutional \underline{N}etwork (SIRGCN),
which incorporates a disease propagation model.
Both RDGCN and SIRGCN are based on reliable and interpretable domain differential equations that
allows the models to generalize to unseen patterns.
We experimentally show that RDGCN and SIRGCN are more robust with mismatched testing data than the state-of-the-art deep learning methods. 
\end{abstract}

\section{Introduction}
Robustness to domain generalization is a crucial aspect in the realm of time series prediction, which has continued to be a topic of great interest, given its myriad uses in many sectors such as transportation~\cite{bui2022spatial}, weather forecasting~\cite{longa2023graph}, and disease control~\cite{jayatilaka2020use}.
When sufficient training data is available where all patterns likely to appear in test situations are represented, deep learning approaches have provided the most accurate predictions.
Among the best-performing deep learning models,
graph-based deep neural networks~\cite{yu2017spatio,wu2020connecting,han2021dynamic,shang2021discrete,ji2022stden} dominate due to their ability to incorporate spatiotemporal information so that dependent information at different locations and times can be captured and exploited to make more accurate predictions.  
However, when collecting 
representative training data is challenging, as it is in many practical situations
since we can only sample in limited conditions, the model trained on such limited data is expected to work in exceptional circumstances.
For example, in the traffic speed prediction problem, natural disasters (e.g., earthquakes or hurricanes) are rare events where traffic patterns can be significantly and abruptly altered. 
At best extreme situations can be simulated, but these cannot truly capture the patterns in an actual event.
Under these circumstances, existing predictive models, especially those deep learning models trained with a large amount of data that are not representative of the testing circumstance, 
do not work well,
i.e., they are not robust to mismatched patterns.

\begin{figure}
\centering
\includegraphics[width=0.99\columnwidth]{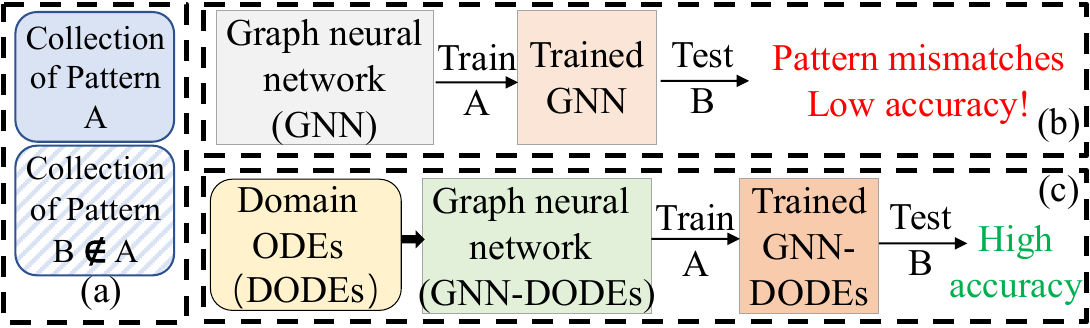} 
\caption{(a) Two collections of patterns (i.e., pattern A exists in a known dataset) and pattern B (difficult to be collected and only available at test time) in the training and test datasets have no overlap; (b) Without incorporating a domain ODE, testing the model with such mismatched patterns may result in poorer accuracy; (c) With an architecture incorporating a domain ODE, 
the model is able to achieve good accuracy given unseen patterns. 
}.

\label{fig:GNNDK-Exp}
\vspace{-0.27in}
\end{figure}

Such challenges are usually classified under domain generalization (Figure~\ref{fig:GNNDK-Exp}), where a model is trained on a source domain but evaluated on a target domain with different characteristics (mismatches).
Consider traffic speed prediction as a motivating example. It is well known that prediction algorithms perform poorly when traffic patterns are unexpectedly disrupted, for instance, due to extreme weather, natural disasters, or even special events. In our evaluation section, we will demonstrate this phenomenon more concretely, where state-of-the-art deep learning methods do not generalize well when dataset patterns are split between training (weekday) and test patterns (weekend). We also provide additional experiments in the appendix to probe further this phenomenon wherein we train a well-known deep learning approach, Spatial-Temporal Graph Convolutional Network (STGCN)~\cite{yu2017spatio}, and use sensitivity analysis
to identify the most influential sensors on the graph that contributed to a particular prediction.
We note that the most influential sensors are geographically nearby when the test dataset is drawn from the same distribution (weekday) as the training data. In contrast, when the test data is drawn from a different distribution (weekend), the three most important sensors identified by the sensitivity analysis are far away. The challenge mentioned above can be formulated as learning with mismatched training data~\cite{9086253},
a problem that is often encountered in practice.
% in different domains.

Our work's driving hypothesis is that if domain equations can capture spatiotemporal dynamics that stay consistent even under data mismatches, then incorporating those equations into the learning methodology can lower the generalization discrepancy of the learned model.
Consequently, the proposed model exhibits robustness to such domain generalization.
To that end, we propose an approach to incorporate domain ordinary differential equations (ODE) into Graph Convolutional Networks (GCNs) for spatiotemporal forecasting, study the generalization discrepancy theoretically, and apply our approach to develop novel domain-informed GCNs for two practical applications, namely traffic speed prediction and influenza-like illness (ILI) prediction. Our experimental results demonstrate that even when the patterns are altered between training and test data, the in-built dynamical relationship ensures that the prediction performance is not significantly impacted. 
Furthermore, the prior knowledge encoded by the domain-informed architecture reduces the number of model parameters, thus requiring less training data.
The model computations are better grounded in domain knowledge and are thus more accessible and interpretable to domain experts.
Our contributions are as follows:
\begin{itemize}[leftmargin=*]  \setlength\itemsep{0em}
    \item We study the challenge of graph time series prediction with mismatched data where the patterns in the training set are not representative of that in the test set. 
    \item We prove theoretically the robustness of domain-ODE-informed GCNs to a particular form of domain generalization when the labeling function differs between the source and target domains. Specifically, we show that the generalization discrepancy is lower for the domain-ODE-informed learning model under certain conditions compared to a domain-independent learning model.
    \item We develop two novel domain-ODE-informed neural networks called
\underline{R}eaction-\underline{D}iffusion \underline{G}raph \underline{C}onvolutional \underline{N}etwork (RDGCN), 
and \underline{S}usceptible-\underline{I}nfectious-\underline{R}ecovered \underline{G}raph \underline{C}onvolutional \underline{N}etwork (SIRGCN) that 
augments GCNs with domain ODEs studied in transportation research~\cite{bellocchi2020unraveling} and disease epidemics~\cite{stolerman2015sir}.
\item Through experimental evaluation on real datasets, we demonstrate that our novel domain-informed GCNs are more robust in situations with data mismatches when compared to baseline models in traffic speed prediction and influenza-like illness prediction. 
%We derive two novel domain-differential-equation-informed machine learning models RDGCN and SIRGCN for robust prediction, and demonstrate the prediction accuracy of domain-ODE-informed GCNs is more robust in situations with data mismatches compared to baseline models, in traffic speed prediction and influenza-like illness (ILI) prediction.
    % \begin{itemize}[leftmargin=*]
    %     \item RDGCN can react more quickly to rapid traffic variations, and impute missing data accurately.
    %     \item SIRGCN can make accurate predictions in the second half of the disease epidemic, even though trained by the data sampled in first half of the disease epidemic.
    % \end{itemize}
\end{itemize}

\section{Related Work}
\noindent\textbf{Graph Neural Networks on Time Series Predictions.}
Graph Neural Networks (GNNs) have been widely utilized to enable great progress in dealing with graph-structured data~\cite{kipf2016semi}.
% One series of studies
~\cite{yu2017spatio,li2017diffusion,cui2020learning} build spatiotemporal blocks to encode the spatiotemporal features. 
% Another series of studies
~\cite{wu2020connecting,shang2021discrete,han2021dynamic,velivckovic2017graph,guo2019attention} generate dependency graphs, which only focus on “data-based” dependency wherein features at a vertex can be
influenced by a vertex, not in its physical vicinity. 
None of these approaches exploit domain ODEs for better generalization and robustness.

\noindent\textbf{Domain generalization.} 
Domain generalization is getting increasing attention recently~\cite{wang2022generalizing, zhou2022domain, robey2021model, zhou2021domain}, and robustness to domain data with mismatched patterns~\cite{sun2023reaction} is important in designing trustworthy models ~\cite{9086253}. 
The goal is to learn a model that can generalize to unseen domains. 
Many works~\cite{robey2021model} assume that there exists an underlying transformation between source and target domain, and use an extra model to learn the transformations~\cite{xian2022mismatched},
therefore the training data must be sampled under at least two individual distributions.
However, our approach addresses the challenge 
by incorporating a domain-specific ODE
instead of using extra training processes learning from the data from two individual domains, or employing additional assumptions on transformations,
thus works for arbitrarily domain scenarios.

\noindent\textbf{Domain differential equations and Neural ODEs.}
Time series are modeled using differential equations in many areas such as chemistry~\cite{scholz2015first} and transportation~\cite{van2015genealogy,loder2019understanding,kessels2019traffic}. 
These approaches focus on equations that reflect most essential relationships. 
To incorporate differential equations in machine learning, many deep learning models based on Neural ODEs~\cite{chen2018neural, jia2019neural, asikis2022neural} have been proposed.
Advancements extend to Graph ODE networks~\cite{ji2022stden, choi2022graph, jin2022multivariate}, which employ black-box differential equations to simulate continuous traffic pattern evolution.  
However, the potential of domain knowledge to fortify algorithmic robustness against domain generalization has yet to be explored.

\section{Problem Definition}
\noindent \textbf{Notations.}
Given an unweighted graph $\mathcal{G} = (\mathcal{V}, \mathcal{E})$ with $|\mathcal{V}|=n$ vertices and $|\mathcal{E}|$ edges,
each vertex $i\in\mathcal{V}$ corresponds to a physical location,
and each edge $(i,j)\in \mathcal{E}$ represents the neighboring connectivity between two vertices.
Let $\mathcal{N}_i$ denote the set of neighbors of vertex $i$,
and $\mathcal{A} \in \mathbb{R}^{n \times n}$ denote the adjacency matrix of the graph $\mathcal{G}$. 
The value of the feature at vertices $i$ at time $t$ is denoted $x_i(t)$, and the vector of features at all vertices at time $t$ is denoted $X(t)$.
Let $X_{t_1:t_2} \in \mathbb{R}^{n \times (t_2-t_1)}$ be the sequence of features $X(t_1),X(t_1+1), \cdots, X(t_2)$ at all vertices in the interval $[t_1,t_2]$.
Assume that the training data and test data are sampled from the source domain $\mathcal{X}_s$ and target domain $
\mathcal{X}_\tau$, respectively.
Data from different domains exhibit different patterns, which we explicitly capture through labeling
functions in each domain. Formally
\small
\begin{equation*}
\begin{aligned}
    \mathcal{X}_s & = \{ (X_{t-T:t}, X_{t+1}):X_{t+1} = l_s(X_{t-T:t}), {X_{t-T:t} \sim \mathcal{D}} \},
\end{aligned}
\end{equation*}
\normalsize
where $l_s$ is the labeling function in the source domain and $\mathcal{D}$ is the distribution of inputs.
The target domain $\mathcal{X}_\tau$ can be defined similarly but with a different labeling function $l_\tau$.
Note that $T$ is the length of the time sequence that defines the ``ground truth'' labeling function, which is usually unknown.
We assume that $T$ is identical in the source and target domains.

\noindent \textbf{Problem definition.}
We aim to solve the problem of single domain generalization~\cite{qiao2020learning, wang2021learning, fan2021adversarially}.
Given the past feature observations denoted as $(X_{t-T:t}^s, X_{t+1}^s) \in \mathcal{X}_s$ on the graph $\mathcal{G}$ on only one source domain $s$, we aim to train a predictive hypothesis $h$ that can 
predict the feature at time $t+1$ for all vertices (denoted as $\hat{X}(t+1) \in \mathbb{R}^{n}$) on the unseen target domain $\tau$
\textbf{Without extra training}. 
We use $L$ to denote a loss function to evaluate the distance between the prediction and ground truth.
% $L(h(X_{t-T:t}),l(X_{t-T:t})): X_{t-T:t} \rightarrow \mathbb{R}$
Let $h$ denote a hypothesis, 
and let $l$ denote the labeling function in the corresponding domain. 
The expectation of the loss is: $\mathcal{L}_{(\mathcal{D}, l)}(h)=\mathbb{E}_{X_{t-T:t} \sim \mathcal{D}}({L}(h(X_{t-T:t}), l(X_{t-T:t})))$.
The hypothesis returned by the learning algorithm is 
\begin{equation}
    h^\ast = \argmin_{h \in \mathcal{H}} \mathcal{L}_{(\mathcal{D}, l_s)}(h),
\end{equation}
where $\mathcal{H}$ is any hypotheses set.
Let $\mathcal{H}^\ast$ denote the set of hypotheses returned by the algorithm, i.e., $\mathcal{H}^* = \{h^\ast: \mathcal{L}_{(\mathcal{D}, l_s)} (h^\ast) < \epsilon\}$, 
and define the discrepancy measure that quantifies the divergence between the source and target domain as~\cite{kuznetsov2016time}:
\begin{equation}
\label{eq:disc}
    disc(\mathcal{H}^*) = \sup_{h\in \mathcal{H}^*} |\mathcal{L}_{(\mathcal{D}, l_s)}(h) - \mathcal{L}_{(\mathcal{D}, l_\tau)}(h)|.
\end{equation}
The objective is to train a hypothesis in the source domain with a low discrepancy to domain generalization. 
We address the challenge by developing GCNs that incorporate domain ODEs, and our methodology is described as follows.

\section{Methodology}
Constructing domain-informed GCNs involves three steps:

\noindent\textbf{\textbullet Define the domain-specific graph.} 
The unweighted graph $\mathcal{G}$ defined earlier should correspond to the real-world network.
Each vertex is associated with a time sequence of data, and edges connect nodes to their neighboring nodes such that the domain equations define the evolution of data at a vertex as a function of the data at 1-hop neighbors.

\noindent\textbf{\textbullet Construct the domain-informed feature encoding function.}
Let $x_i(t)$ denote a feature at vertex $i$ at time $t$ ,and $H^i_{t,T}$ denote the length $T$ history of data prior to time $t$, and set $\mathcal{N}_i$ of 1-hop neighbors of vertex $i$. The ODE models the feature dynamics at vertex $i$ is given by
\begin{equation}
\label{eq:domain-ode}
    \frac{dx_i(t)}{dt} = f_i(x_i(t), \{x_j(t)|j\in \mathcal{N}_i\}) + g_{i}(H^i_{t,T}), 
\end{equation}
where $f_i(x_i(t),\{ x_j(t) | j \in \mathcal{N}_i \})$ models the evolution of feature~\cite{asikis2022neural, xhonneux2020continuous} at vertex $i$ as a dynamic system related only to the feature at vertex $i$ and the neighboring vertices at current time.
Among other things, $f$ encapsulates the invariant physical properties of the network.
For example, in transportation networks, demand patterns might change, but traffic flow dynamics would not. In disease transmission, travel patterns might change, but the dynamics of infection transmission would not.
In Eq. (\ref{eq:domain-ode}), the influence of the temporal patterns on the measurement cannot be captured by the immediate
dynamics is captured by the function $g_i$, which takes the feature history in a $T-$length window as input.
%$g_{i}(\{ X_{\alpha, k} | \alpha=t-T:t, k\in \mathcal{V} \} \setminus \{X_{t, j} | j \in \mathcal{N}_i \}$ 
In Eq. (\ref{eq:domain-ode}), $g_i(H^i_{t,T})$ is the function of the feature history in a $T-$length window prior to time $t$ and data from non-adjacent vertices at time $t$.
The pattern-specific function $g_i$ is used to capture some impact of the past data\footnote{E.g., the congestion is caused by the increasing traffic demand. 
% a temporary increase in infection rates because of increasing temperatures
} and the impact from distant vertices\footnote{E.g., temporary change of travel demand.}. 
The ODE models immediate dynamics are widely studied in many domains. Thus, the function $f_i$ is usually considered a known function~\cite{maier2019learning}, while the pattern-specific function $g_i$ is difficult to capture, and is generally considered an unknown function. 
Due to an unknown function, most of the existing deep learning models design complex architectures to approximate $g_i$. 
However, in our context, we assume mismatched patterns between domains mentioned earlier lead to the difference between the pattern-specific function $g_i$ in the source and target domain, i.e., let $g_{s, i}$ and $g_{\tau, i}$ denote the pattern-specific function at vertex $i$ in source and target domain respectively. The difference between the labeling function in the respective domain (i.e., $l_s$ and $l_\tau$) is caused by
\small
\begin{equation}
\label{eq:domain-diff}
    g_{s, i}(H^i_{t,T}) \neq g_{\tau, i}(H^i_{t,T}).
\end{equation}
\normalsize
To mitigate the effects of such pattern mismatches in Eq. (\ref{eq:domain-diff}), we propose the GCN incorporating domain ODEs, which is a family of GCNs that incorporate the domain equations $f_i$ to learn only the immediate dynamics to be robust to the domain generalization. 
We employ a feature extraction function, $O$, to encode inputs by selecting the relevant input by utilizing a domain graph:
\begin{equation}
    O(X(t), \mathcal{A}) = \mathcal{A} \otimes X(t),
\end{equation}
where $\otimes$ is the Kronecker product, and $\mathcal{A}$ is the adjacency matrix of graph $\mathcal{G}$.
% which is used to encode features required by the domain equations. 
We then generalize the local domain Eq. (\ref{eq:domain-ode}) to a graph-level representation:
\small
\begin{equation}
    \frac{dX(t)}{dt} = F(O(X(t), \mathcal{A});\Theta_1) + G(O(X(t), \mathbb{I} - \mathcal{A}), X_{t-T:t-1};\Theta_2),
\end{equation}
\normalsize
where $\mathbb{I}$ is the all-one matrix, $F$ (resp. $G$) with parameters $\Theta_1$ (resp. $\Theta_2$) is a collection of $\{f_1, ... f_n\}$ (resp. $\{g_1, ..., g_n \}$) of the encoded domain-specific features. 

\noindent\textbf{\textbullet Define the Network Prediction Function.}
The domain-ODE-informed GCNs only learn $F$. Thus a network-level prediction using the finite difference method is:
\small
\begin{equation}
    \label{eq:ode-prediction}
    \begin{aligned}
        \hat{X}({t+1}) & = X(t) + \int_t^{t+1} F(O(X(t), \mathcal{A});\Theta_1) dt \\
        & \approx X(t) + F(O(X(t), \mathcal{A});\Theta_1).        
    \end{aligned}
\end{equation}
\normalsize

\section{Proof of Robustness to Domain Generalization}
Without incorporating domain ODEs, most GNNs need longer data streams to make accurate predictions. For instance, black-box predictors in the traffic domain
require $12$ time points to predict traffic speeds, whereas the domain informed GCN we develop requires only $1$ time point as it explicitly incorporates the immediate dynamics instead of learning arbitrary functions. (see Eq.(\ref{eq:domain-knowledge})).
%and the speeds at any time are connected with the speed of the previous time step through the RD equations. 
We will discuss the application-specific GCNs in the subsequent section. In the following, we will prove that when the underlying dynamics connect the features at consecutive time points, the approach
that incorporates the dynamics is more robust to the domain generalization problem defined by the discrepancy equation in Eq. (\ref{eq:disc}).
Similar to the approach in~\cite{redko2020survey}, we assume the training set is sampled from the source domain, and the test data is sampled from the target domain. 
In this work, we formulate the mismatch problem as a difference between labeling functions in the source and target domains where the immediate time and nearest neighbor dynamics (function $F$) are unchanging across domains. In contrast, the impact of long-term and distant neighbor patterns (function $G$) varies between source and target domains. 
We observe that although both $G_s$ (resp. $G_\tau$) and $F$ utilize $X(t)$ as part of their input, they consistently select features from distinct nodes. Thus there is no overlap between inputs of $G_s$ (resp. $G_\tau$) and $F$.

Under such a mismatch scenario,
we will show that using long-term patterns and data from nodes outside the neighborhood will have worse generalization as measured by a discrepancy function. 
We use $\mathcal{H}_1$ to denote the hypothesis set mentioned earlier that predicts the data at time $t+1$ based on a $T$-length history (from $t-T$ to $t$, where $T>1$), and
$\mathcal{H}_2$ denotes the hypothesis set that uses the data only at time $t$ to predict the speed at $t+1$. In other words, baseline algorithms that use several time points and data from nodes outside the 1-hop neighborhood would fall into $\mathcal{H}_1$. In contrast, algorithms such as ours, which use domain ODEs to incorporate the known functional form $F$, which requires only immediate and nearest neighbor data, would belong to $\mathcal{H}_2$.
We make the following two assumptions:
\begin{assumption}
\label{assump:realizability}
There exists $h_1^* \in \mathcal{H}_1$ s.t. $\mathcal{L}_{(\mathcal{D}, F+G_s)}(h_1^\ast) = 0$. There exists $h_2^* \in \mathcal{H}_2$ s.t. $\mathcal{L}_{(\mathcal{D}, F)}(h_2^\ast) = 0$. 
\end{assumption}
\begin{assumption}
\label{assump-symmetric}
Let $U=G_s(O(X(t), \mathbb{I}-\mathcal{A}), X_{t-T:t-1};\Theta_2)$ be a random variable where ${X_{t-T:t} \sim \mathcal{D}}$ and  $P_U(G)$ be the probability distribution function (PDF) of $U$. The PDF $P_U(G)$ is symmetric about $0$.
\end{assumption}

Assumption~\ref{assump:realizability} ensures the learnability of the hypotheses.
Assumption~\ref{assump-symmetric} ensures that the statistical impact of the long-term pattern is unbiased and symmetric. 
In the appendix, we show the used datasets satisfy these assumptions. 
The above assumptions lead to the following Lemmas about optimal hypotheses learned by domain-agnostic methods, such as the baselines, and those learned by our domain-informed methods, such as ours.
\begin{lemma}
\label{lemma:train-h2}
%Assumption \ref{assump:realizability} implies that
$h_1^\ast(X_{t-T:t}) = F(O(X(t), \mathcal{A});\Theta_1)+G_s(O(X(t), \mathbb{I} - \mathcal{A}), X_{t-T:t-1};\Theta_2)$.
\end{lemma}
\begin{proof}
Follows Assumption \ref{lemma:train-h2} when $\mathcal{L}_{(\mathcal{D}, l_s)}(h_1^\ast) = 0$.
\end{proof}

\begin{lemma}
\label{lemma:f=h}
If (1) $h_2$ is trained with data sampled from $\mathcal{X}_s$ such that assumption~\ref{assump-symmetric} is true, (2) the loss function ${L}$ is the L1-norm or MSE,  
then $h_2^\ast = F$.
\end{lemma}

To theoretically establish the enhanced robustness of our approach, we assume the PAC learnability of $\mathcal{H}_1$ and $\mathcal{H}_2$.
% We assume both $\mathcal{H}_1$ and $\mathcal{H}_2$ are PAC learnable. 
In detail, with sufficient data, for every $\epsilon_1, \epsilon_2, \delta \in (0,1)$, if Assumption~\ref{assump:realizability} holds with respect to $\mathcal{H}_1, \mathcal{H}_2$, then when running the learning algorithm using data generated by distribution $\mathcal{D}$ and labeled by $F+G_s$, with the probability of at least $1-\delta$, the hypothesis $h_1^\ast$ is in the set
\begin{equation}
    \begin{aligned}
        \mathcal{H}_1^* = \{h_1^\ast: \mathcal{L}_{(\mathcal{D}, F+G_s)} (h_1^\ast) < \epsilon_{1}\},
    \end{aligned}
\end{equation}
and with the probability of at least $1-\delta$, $h_2^\ast$ is in the set
% \begin{equation}
%     \mathcal{L}_{(\mathcal{D}, f)} (h_2^\ast) < \epsilon_{2},
% \end{equation}
\begin{equation}
    \begin{aligned}
        \mathcal{H}_2^* = \{h_2^\ast: \mathcal{L}_{(\mathcal{D}, F)} (h_2^\ast) < \epsilon_{2}\}.
    \end{aligned}
\end{equation}
We will now demonstrate that $\mathcal{H}_2^*$ is more robust to the domain generalization than $\mathcal{H}_1^*$ using the discrepancy measure defined in Eq. (\ref{eq:disc}). 
For our theoretical result, we require that the loss function $L(h,l)$ satisfy triangle inequality:
\begin{equation}
\label{eq:triangle-equality}
    |L(h, h') - L(h', l)| \leq L(h, l) \leq L(h, h') + L(h', l),
\end{equation}
where $h'$ is any other hypothesis. 
The following theorem proves our result.
\begin{theorem}
\label{theorem:robust}
If (1) the training data is sampled from the source domain where assumption \ref{assump-symmetric} is true, (2) the loss function ${L}{(h, l)}$ obeys the triangular inequality, 
% (3) $\epsilon_1, \epsilon_2 \leq \min(\mathbb{E}_{X_{t-T:t} \sim \mathcal{D}}[|g_s(\{x_{t,i}, x_{t, j} | j \notin \mathcal{N}_i \}))|], \mathbb{E}_{X_{t-T:t} \sim \mathcal{D}}[|g_\tau(\{x_{t,i}, x_{t, j} | j \notin \mathcal{N}_i \})|])$, 
then the discrepancy should satisfy
\begin{equation}
    disc(\mathcal{H}_2^\ast) \leq disc(\mathcal{H}_1^\ast).
\end{equation}
\end{theorem} 

Theorem~\ref{theorem:robust} illustrates that models trained using lengthy time sequences and distant nodes are not reliable when there are mismatches between 
the labeling functions in the source domain and target domain.
Loss functions that include mean absolute error and root mean squared error satisfy the triangle inequality assumption. We note that this assumption is not satisfied by the mean-squared error (MSE) loss function. In the Appendix, we prove an extension of this theorem for MSE by making an additional assumption on the data.

\section{Application of Domain-ODE informed GCNs}

In the following part of this section, we will use the Reaction-Diffusion equation and SIR-network differential equation as examples to develop practical domain-informed GCNs. 

\noindent\textbf{Reaction Diffusion GCN for Traffic Speed Prediction.}
Bellocchi in ~\cite{bellocchi2020unraveling} proposed the reaction-diffusion approach to reproduce traffic measurements such as speed and congestion using few observations. The domain differential equations included a \textit{Diffusion} term that tracks the influence in the direction of a road segment, while the \textit{Reaction} term captures the influence opposite to the road direction.
Since each sensor is placed on one side of a road segment and measures the speed along that specific direction,
$\mathcal{A}$ is asymmetric, and in particular, only one of $\mathcal{A}_{i,j}$ and $\mathcal{A}_{j,i}$ can be non zero. 
Consider sensor $i$, let $\mathcal{N}_i^d$ denote the set of sensor $i$'s neighbors in the road segment direction, and 
let $\mathcal{N}_i^r$ denote the set of the neighbors in the opposite direction of the sensor $i$. 
If $x_i(t)$ denotes the speed observed at node $i$ at time $t$, the local reaction-diffusion equation\footnote{See the appendix
% ~\ref{sec:rdgcn} 
for more details.} at vertex $i$ can be formulated as
% \vspace{-.03in}
\small
\begin{equation}\label{eq:domain-knowledge}
% \vspace{-.08in}
\begin{aligned}
     \frac{d x_i(t)}{d t} = & \sum_{j\in \mathcal{N}^d} \rho_{(i, j)} \left(x_{j}(t) - x_{i}(t)\right) + b^{d}_{i}
   \\ + & \tanh \left(\sum_{j\in \mathcal{N}^r} \sigma_{(i, j)} (x_{j}(t) - x_{i}(t)) + b^{r}_{i} \right), 
\end{aligned}
\end{equation}
\normalsize
where $\rho_{(i, j)}$ and $\sigma_{(i, j)}$ are the diffusion parameter and reaction parameter, respectively; 
% $b^{d}_{i}$ is the diffusion bias which corresponds to the estimation of the metastability of the diffusion process at vertex $i$, 
$b^{d}_{i}$ and $b^{r}_{i}$ are biases to correct the average traffic speed at vertex $i$ in diffusion and reaction.

In the following, we incorporate this reaction-diffusion (RD) equation using the steps outlined in Section Methodology to build a novel GCN model for domain-informed prediction of traffic speed.% by constructing the RD function. 

% \noindent\textbf{$\protect\circled{1}$} 
\textbf{Step 1: Define reaction and diffusion parameters.}
% \noindent\textbf{$\protect\circled{1}$ Derive the Adjacency Matrices from the Physical Graph for the Reaction and Diffusion Process.}
We define a diffusion graph $\mathcal{G}^d = (\mathcal{V}, \mathcal{E}^d)$ and a reaction graph $\mathcal{G}^r = (\mathcal{V}, \mathcal{E}^r)$ derived from the physical graph $\mathcal{G}$. The diffusion graph represents whether two vertices are direct neighbors in the road direction, i.e., $\mathcal{E}^d = \mathcal{E}$ and $\mathcal{A}^d = \mathcal{A}$; the reaction graph represents whether two vertices are direct neighbors in the opposite direction of a road segment, i.e., $\mathcal{E}^r = \{ (i,j): (j,i) \in \mathcal{E} \}$ and $\mathcal{A}^r = \mathcal{A}^\top$, where $\top$ denotes matrix transpose.
% \noindent\textbf{$\protect\circled{2}$ Define Model Weights for Reaction and Diffusion Networks based on the Physical Equation.}
Define $\mathbf{\rho} = \{\rho_{(i,j)} \in \mathbb{R} | {(i,j)}\in\mathcal{E}^d\}$, $\mathbf{\sigma} = \{\sigma_{(i,j)} \in \mathbb{R} | {(i,j)}\in\mathcal{E}^r\}$, $b^d \in \mathbb{R}^{n} $, $ b^r \in \mathbb{R}^{n}$. 
Each parameter $\rho_{(i,j)}$ (resp. $\sigma_{(i,j)}$) is a diffusion weight (resp. reaction weight) for edge $(i,j)$. Each parameter in $\mathbf{\rho}$ and $\mathbf{\sigma}$ corresponds to a directed edge $(i,j)$ in $\mathcal{E}^d$ and $\mathcal{E}^r$, respectively. 
% Let ${\Theta}_1 \in \mathbf{R}^{n \times n}$ and ${\Theta}_2 \in \mathbf{R}^{n \times n}$ denote the sparse parameter metrics after setting the value of non-zeros in $A^{d}$ and $A^{r}$ as the value of corresponding element in $\theta_1$ and $\theta_2$  respectively. 
% Let ${\Theta}_1 \in \mathbf{R}^{n \times n}$ and ${\Theta}_2 \in \mathbf{R}^{n \times n}$ denote the sparse parameter adjacency metrics after mapping each element in $\theta_1$ and $\theta_2$ to each non-zeros in $A^{d}$ and $A^{r}$ respectively. 
$\mathbf{W}^d \in \mathbb{R}^{n \times n}$ 
% and $\mathbf{W}^r \in \mathbb{R}^{n \times n}$ 
is a sparse weight matrix for diffusion graph $\mathcal{G}^d$, 
where $\mathbf{W}_{i,j}^d=\rho_{(i,j)}, \forall (i,j) \in \mathcal{E}^d$, 
otherwise $\mathbf{W}^d_{i,j}=0$.
% \begin{eqnarray}
% % \vspace{-.18in}
% \mathbf{W}_{i,j}^d=
% \begin{cases}
%  \rho_{(i,j)} & \text { $\forall (i,j) \in \mathcal{E}^d$ }, \\ 
%  0 & \text { otherwise, }
% \end{cases}
% \label{eq:weight-diff}
% \end{eqnarray}
$\mathbf{W}^r$ for reaction graph $\mathcal{G}^r$ is defined in a similar way, but the non-zero element at $(i,j) \in \mathcal{E}^r$ is $\sigma_{(i,j)}$.
% where $\forall (i,j) \in \mathcal{E}^r$.
% \begin{eqnarray}\mathbf{W}_{i,j}^r=
% \begin{cases}
%  \sigma_{(i,j)} & \text { $\forall (i,j) \in \mathcal{E}^r$ }, \\ 
%  0 & \text { otherwise. }
% \end{cases}
% \label{eq:weight-react}
% \end{eqnarray}

% In the traffic speed prediction, the physical graphs roughly indicate whether one vertex is reachable from another vertex within 5 minutes for both the reaction and diffusion. 

% \noindent\textbf{$\protect\circled{3}$ Characterize the Graph Laplacian by Combining the Adjacency Matrices and the Defined Parameters for the Reaction and Diffusion Terms.}

% Given any weighted adjacency matrix $\mathbf{A}$, the corresponding Laplacian matrix $\mathbf{L}$ can be calculated by the function ${L}ap(\mathbf{A})$ which is defined as
% % \begin{equation}
% %     \mathbf{L} = {L}ap(\mathbf{A}) = Diagonal(\sum_{j=1}^{n} \mathbf{A}_{i,j}) - \mathbf{A}.
% % \end{equation}
% \begin{equation}
%     \mathbf{L} = {L}ap(\mathbf{A}) = Degree(\mathbf{A}) - \mathbf{A},
% \end{equation}
% where $Degree(*)$ is to calculate the degree matrix of an input adjacency matrix. In this work, we use a common variation measure in graph signal processing~\cite{chung1997spectral} to express the action of the Laplacian on sensor $i$ and $X_t$
% \begin{equation}
% \label{eq:linear-laplacian}
%     (\mathbf{L}X_t)_i = ((Degree(A)-A)X_t)_i = \sum_{(i,j) \in \mathcal{E}} \mathbf{A}_{i,j} (X_{t,j} - X_{t,i}).
% \end{equation}
% % where $(i, j)$ is the directed edge from $i$ to $j$.

% \noindent\textbf{$\protect\circled{2}$}
\textbf{Step 2: Construct RD feature encoding function.}
Let $\mathbf{L}^d$ (resp. $\mathbf{L}^r$) be the corresponding Laplacian of the combination of diffusion (resp. reaction) weight tensor $\mathbf{W}^d$ (resp. $\mathbf{W}^r$) and diffusion (resp. reaction) adjacency matrices $\mathcal{A}^d$ (resp. $\mathcal{A}^r$),
% and let $\mathbf{L}^r$ be the Laplacian of the combination of reaction weight tensor $\mathbf{W}^r$ and reaction adjacency matrices $\mathcal{A}^r$, 
then
\small
\begin{equation}
\label{eq:linear-laplacian-reaction-diffusion}
    \begin{aligned}
    (\mathbf{L}^d X(t))_i & = \sum_{(i,j) \in \mathcal{E}^d} (  \mathbf{W}^d \odot \mathcal{A}^d)_{i,j} (X_j(t) - X_i(t)) \\
    & = ((Degree(\mathbf{W}^d \odot \mathcal{A}^d)-\mathbf{W}^d \odot \mathcal{A}^d)X(t))_i,  
    % (\mathbf{L}^r X_t)_i & = \sum_{(i,j) \in \mathcal{E}^r} ( \mathbf{W}^r \odot \mathcal{A}^r)_{i,j} (X_{t,j} - X_{t,i}),
    \end{aligned}
\end{equation}
\normalsize
where $\odot$ denotes the Hadamard product, $Degree(*)$ is to calculate the degree matrix of an input adjacency matrix, and $(\mathbf{L}^r X_t)_i$ represents a similar reaction process but the weight tensor is $\mathbf{W}^r$ and Adjacency matrix is $\mathcal{A}^r$.
Specifically, the reaction and diffusion laplacian $\mathbf{L}^r$ and $\mathbf{L}^d$ is the RD-informed feature encoding function $O$ extracting speed differences between neighboring vertices.  

% \noindent\textbf{$\protect\circled{3}$} 
\textbf{Step 3: }
Using Eq~(\ref{eq:ode-prediction}) we can define a prediction:
\small
\begin{equation}
\label{eq:network-prediction}
    \hat{X}({t+1}) = X(t) + (\mathbf{L}^d X(t) + b^{d}) + \tanh{(\mathbf{L}^r X(t) + b^{r})},
\end{equation}
\normalsize

where $\mathbf{L}^d$ and $\mathbf{L}^r$ is the reaction and diffusion functions constructed earlier, 
corresponds to the function $F=(\mathbf{L}^d X_t + b^{d}) + \tanh{(\mathbf{L}^r X_t + b^{r})}$ predicting the traffic speed using
the reaction parameters $\rho$ and the diffusion parameters $\sigma$.
% $F$ is $(\mathbf{L}^d X_t + b^{d}) + \tanh{(\mathbf{L}^r X_t + b^{r})}$ in RDGCN.
% As shown in Figure~\ref{fig:architecture}, we define a GNN model to predict traffic speed with Eq.~(\ref{eq:weight-diff}), ~(\ref{eq:linear-laplacian-reaction-diffusion}), and~(\ref{eq:network-prediction}).

\textbf{Susceptible-Infected-Recovered (SIR)-GCN for Infectious disease prediction.}
The SIR model is a typical model describing the temporal dynamics of an infectious disease by dividing the population into three categories: Susceptible to the disease, Infectious, and Recovered with immunity. The SIR model is widely used in the study of diseases such as influenza and Covid~\cite{cooper2020sir}. 
Our approach is based on the SIR-Network Model proposed to model the spread of Dengue Fever~\cite{stolerman2015sir}, which we describe as follows.
Let ${S}_i(t)$, ${I}_i(t)$, ${R}_i(t)$ denote the number of Susceptible, Infectious, and Recovered at vertex $i \in \mathcal{V}$ at time $t$ respectively and the total population at vertex $i$ is assumed to be a constant, i.e., $N_i = S_i(t) + I_i(t) + R_i(t)$.

% \noindent\textbf{$\protect\circled{1}$} 
\textbf{Step 1: Define the travel matrices.} The spread of infection between nodes is modeled using sparse travel matrices $\mathbf{\Phi} \in [0, 1]^{n \times n}$ as
$\mathbf{\phi}{(i,j)}, \forall (i,j) \in \mathcal{E}^d$; otherwise $\mathbf{\phi}{(i,j)}=0$,
% \begin{eqnarray}
% % \vspace{-.18in}
% \mathbf{\Phi}_{i,j}^d=
% \begin{cases}
%  \mathbf{\phi}{(i,j)} & \text { $\forall (i,j) \in \mathcal{E}^d$ }, \\ 
%  0 & \text { otherwise, }
% \end{cases}
% \label{eq:weight-diff}
% \end{eqnarray}
where  $\phi_{(i,j)} \in [0, 1]$ is a parameter representing the fraction of resident population travel from $i$ to $j$,
therefore we require the fractions satisfy $\sum_{k=1}^n \phi_{(i,j)} = 1, \forall i \in \mathcal{V}.$
% \begin{equation}
%     \sum_{k=1}^n \phi_{(i,j)} = 1, \forall i \in \mathcal{V}.
% \end{equation}
% Moreover, we define the total population travel from all vertices to vertex $i$ as
% \begin{equation}
%     \label{eq:total-travel-population}
%     N_i^p = \sum_{k=1}^n \phi_{(i,k)} N_k,
% \end{equation}
% which is also constant in time.
The SIR-network model at vertex $i$ is defined as:
\small
\begin{equation}
\label{eq:network-level}
\begin{aligned}
    \frac{d{S}_i(t)}{dt} = & -\sum_{j=1}^{M}\sum_{k=1}^{M} \beta_j \phi_{(i,j)}S_i(t)\frac{\phi_{(k, j)}I_k(t)}{N_j^p}, \, \\
    \frac{d{I}_i(t)}{dt} = & \sum_{j=1}^{M}\sum_{k=1}^{M} \beta_j \phi_{(i,j)}S_i(t)\frac{\phi_{(k, j)}I_k(t)}{N_j^p} - \gamma I_i, \, \\
    \frac{d{R}_i(t)}{dt} = & \gamma I_i(t),
    % N_i^p &= \sum_{k=1}^n \phi_{(i,j)} N_k,
\end{aligned}
\end{equation}
\normalsize
where $\beta_i$ is the infection rate at vertex $i$, representing the probability that a susceptible population is infected at vertex $i$,  and $\gamma$ is the recovery rate, representing the probability that an infected population is recovered, $N_i^p = \sum_{k=1}^n \phi_{(i,j)} N_k$ is the total population travel from all vertices to vertex $i$. We assume the recovery rates at all vertices are the same. 

% \noindent\textbf{$\protect\circled{2}$} 
\textbf{Step 2: Construct the SIR function.} The differential equation system~(\ref{eq:network-level}) is equivalent to:
\small
\begin{equation}
\label{eq:sirgnn-prediction}
    \frac{dI(t)}{dt} = (\mathcal{K} - \gamma) I({t}),
\end{equation}
\normalsize
where $I(t)$ is the feature ($X(t)$ mentioned earlier) representing the number of Infectious people. Then the transformation matrix $\mathcal{K}$ connecting $I(t)$ and $I(t+1)$ at neighboring time is:
\small
\begin{equation}
\label{eq:K}
    \mathcal{K}_{i,j} = \sum_{j=1}^n \beta_j \phi_{(i,j)}\phi_{(k,j)} \frac{S_i}{N_j^p},
\end{equation}
\normalsize
where $S_i(t) = N_i - I_i(t) - R_i(t)$ and $R_i(t) = \gamma \int_{t_0}^{t} I(t) = \gamma \sum_{t_0}^{t} I(t)$, $t_0$ is the starting time of the current epidemic.
The domain-informed feature encoding function $O$ is utilized to approximate the counts of susceptible and recovered populations and estimate the infectious people likely to travel, approximated by the flight data.

% \noindent\textbf{$\protect\circled{3}$} 
\textbf{Step 3: }
Using Eq.~(\ref{eq:ode-prediction}), (\ref{eq:sirgnn-prediction}), and (\ref{eq:K}), the prediction is
\small
\begin{equation}
    \hat{I}({t+1}) = {I}(t) + (\mathcal{K} - \gamma) I(t).
\end{equation}
\normalsize

% \noindent\textbf{$\protect\circled{1}$ Derive the number of recovered cases and the number of Susceptible cases at all vertices.} In this step, we approximate the recovered population $R(t)$ using finite difference method:
% \begin{equation}
%     R(t) = \gamma \sum_{t_0}^{t} I(t).
% \end{equation}
% Then the Susceptible cases can be approximated using
% \begin{equation}
%     S(t) = N - I(t) - R(t),
% \end{equation}
% where $N$ is the total number of populations given in the dataset. 

% \noindent\textbf{$\protect\circled{2}$ Calculating the total number of population travel to each vertex and construct $\mathcal{K}$.} 

% \noindent\textbf{$\protect\circled{3}$ Network prediction using Eq~(\ref{eq:sirgnn-prediction}).} 

\section{Evaluation}
In this section, we compare the performance of these domain-ODE-informed GCNs with baselines when tested with mismatched data and demonstrate that our approach is more robust to such mismatched scenarios.

\begin{figure*}[!h] 
    \centering
    \vspace{-.03in}
    \subfloat[Baseline models and RDGCN trained on $12$ consecutive weekdays and all models are augmented by MAML]
    {\label{fig:mismatch}
    % \hspace{-.0in}\includegraphics[width=0.99\textwidth, height=0.266\textwidth]{{figs/mismatch/mismatches-maml.pdf}}}
    \hspace{-.0in}\includegraphics[width=0.95\textwidth]{{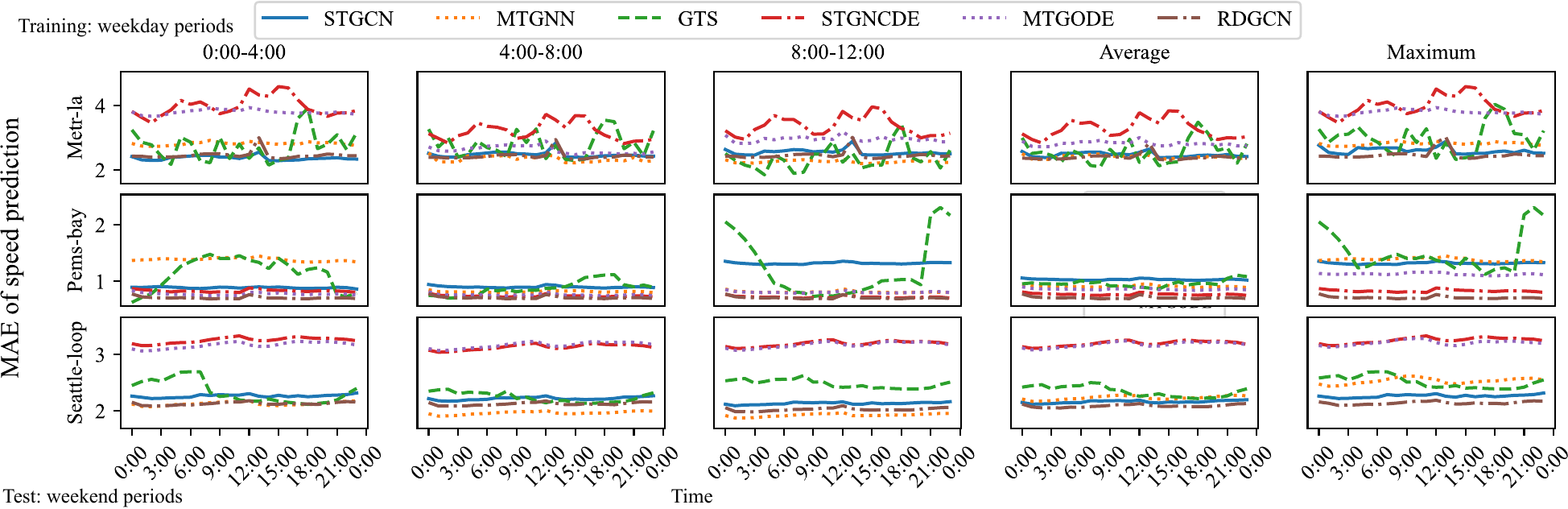}}}
    %\caption{fig1}
    % \vspace{-.08in}
    \hfill%
    \subfloat[Baseline models and RDGCN trained on more than half a year of weekdays]
    {\label{fig:mismatch-maml}
    \vspace{-0.5in}
    \hspace{-.03in}\includegraphics[width=0.95\textwidth, 
    % height=0.266\textwidth
    ]{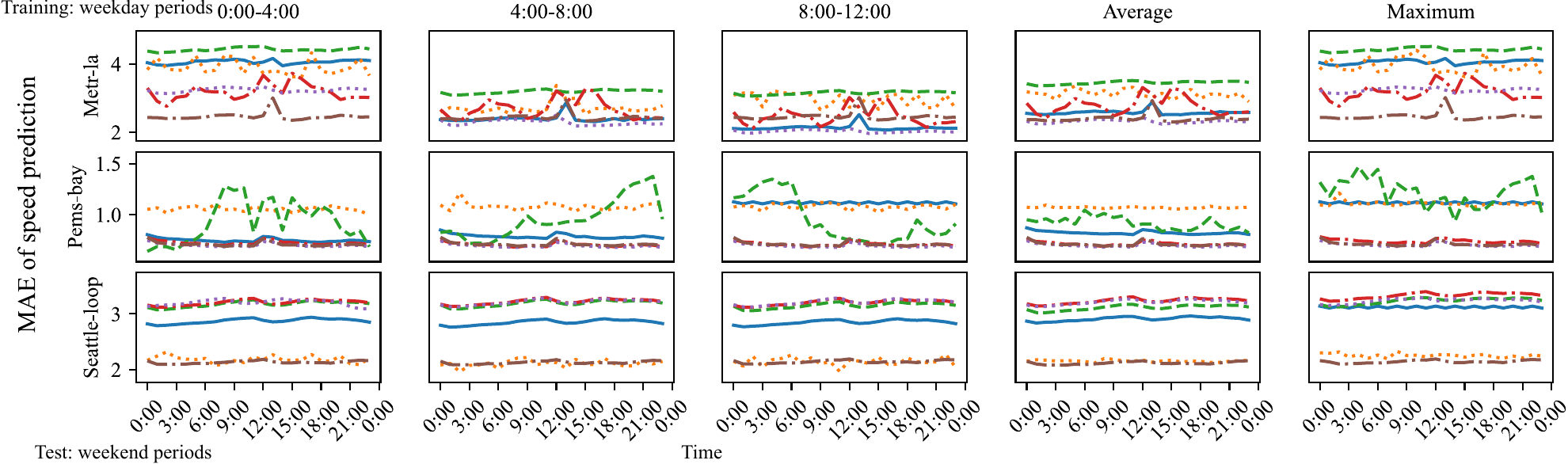}}
    % \hspace{-.03in}\includegraphics[width=0.99\textwidth, height=0.266\textwidth]{figs/mismatch/mismatches-full.pdf}}
    \vspace{-.08in}
    % \subfloat[MSE: Limited and mismatched data with MAML augmentation.]
    % {\label{fig:mismatch-mse}
    % \hspace{-.0in}\includegraphics[width=0.99\textwidth, height=0.266\textwidth]{{figs/mismatch/mismatches.pdf}}}
    % %\caption{fig1}
    % \vspace{-.08in}
    % \hfill%
    % \subfloat[MSE: Sufficient mismatched data.]
    % {\label{fig:mismatch-maml-mse}
    % % \vspace{-0.5in}
    % \hspace{-.03in}\includegraphics[width=0.99\textwidth, height=0.266\textwidth]{figs/mismatch/mismatches-maml.pdf}}
    % \vspace{-.08in}
    \caption{
(a) The results of RDGCN are very close regardless of the period of the training set. 
(b) Even though all the models are trained using all available weekdays, the results of RDGCN are still closer regardless of the period, compared to baseline models.
The numerical result, the plot in the other three time windows, and the corresponding result for RMSE are in Ablation Study in the appendix.    }
\vspace{-.08in}
\label{fig:mismatch-all}
\end{figure*}

\subsection{Experiment Settings}
\noindent\textbf{Datasets.}
Our experiments are conducted on three real-world datasets (Metra-la, Pems-bay and Seattle-loop) for traffic prediction, and two real-world datasets (in Japan and US) for disease prediction. The details are shown in Table \ref{tab:dataset_stats}.

\begin{table}[htbp]
    \centering
    \small
        \caption{Dataset statistics}
        \vspace{-0.1in}
\resizebox{0.99\columnwidth}{!}{
    \begin{tabular}{c|c c c c c}
    \toprule
        Dataset & $|\mathcal{V}|$ & $|\mathcal{E}|$ & resolution & period  \\
    \midrule
        Metr-la \cite{jagadish2014big} & 207 & 233 & 5 mins & 122 days  \\
        Pems-bay \cite{li2017diffusion} & 281 & 315 & 5 mins & 151 days  \\
        Seattle-loop \cite{cui2020learning} & 323 & 660 & 5 mins & 365 days \\
        Japan-Prefectures \cite{deng2020cola} & 47 & 133 & weekly & 347 weeks  \\
        US-States \cite{deng2020cola} & 49 & 152 & weekly & 834 weeks  \\
    \bottomrule
    \end{tabular}
    }
    \label{tab:dataset_stats}
\end{table}
\vspace{-0.1in}

\noindent\textbf{Evaluation Metric.}
The loss function we use is the mean absolute error and the root mean square error:
% \begin{equation}
%     \label{eq:MAE-loss}
    $MAE(X(t), \hat{X}(t)) = \frac{1}{n} \sum_{i=1}^{n} |x_i(t) - \hat{x}_i(t)|$,     $RMSE(X(t), \hat{X}(t)) = (\frac{1}{n} \sum_{i=1}^{n} (x_i(t) - \hat{x}_i(t)))^2)^{\frac{1}{2}}$.
% \vspace{-0.0in}
% \end{equation}
We also use MAE and RMSE to evaluate models. 
We note that both these loss functions satisfy the triangle inequality.
% We assume that all zeros in the datasets are missing values, and we remove the predicted speed of the sensors in the period when the ground truth is $0$, or when the last speed recorded is $0$. 

% \noindent\textbf{Hyperparameter Settings.} 
% RDGCN is optimized via Adam. The batch size is set as 64. The learning rate is set as $0.001$ and the early stop strategy is used with patient of 30 epochs. The training and validation set are split by ratio of 3:1 from the weekday subset and the test data is sampled from the weekend subset with different patterns.

\noindent\textbf{Baselines.}
% Traffic speed prediction has been studied in literature as a function of multi-dimensional data, such as traffic flow, weather conditions, and other factors to assist historical traffic speed.
For traffic prediction tasks,
we compare RDGCN with STGCN~\cite{yu2017spatio}, MTGNN~\cite{wu2020connecting}, GTS~\cite{shang2021discrete}, STGNCDE~\cite{choi2022graph}, and MTGODE~\cite{jin2022multivariate}.
They are influential and best-performing deep learning models for predicting traffic speed using historical speed alone.
We also use Model-Agnostic Meta-Learning (MAML)~\cite{finn2017model} to help baseline models, and our approach adapts quickly to tasks using good initial weights generated by MAML. 
For disease prediction, we compare SIRGCN with two state-of-the-art models for infection prediction, ColaGNN~\cite{deng2020cola} and EpiGNN~\cite{xie2023epignn}.

% They are best-performing deep learning models for predicting number of infectious population.
% for all baseline models and our approach that can be trained with a gradient descent procedure.
% The support set and the query set are equally divided from the training data. The learning rate for the inner loop is $0.00005$, for the outer loop is $0.0005$, and MAML is trained for $200$ epochs. 
% Our experiment involves the following steps: 1.We divide the training set into two equal parts: the support set and the query set. 2.We use the support set to compute adapted parameters. 3.We use the adapted parameters to update the MAML parameters on the query set. 4.We repeat this process $k=200$ times to obtain initial parameters for the baseline model. 5.We train the baseline model using the obtained initial parameters. }

\subsection{Results and analysis} \label{sec:result}
\textbf{Mismatched Data Experiments for RDGCN.}
We first explore the performance of the models when they are trained using mismatched data from certain conditions and tested using alternate, mismatched conditions.
Specifically, 
the models are trained with four-hour data on weekdays (e.g.,8:00-12:00 on weekdays) selected and evaluated with hourly data
on weekends (e.g., 13:00-14:00 on weekends). 
In limited data and mismatched conditions (Figure \ref{fig:mismatch}), the training
set consists of data from five different sequences of 12 consecutive
weekdays selected randomly from the available data. 
This experiment aims to replicate scenarios where data collection is challenging, and traffic patterns undergo rapid changes.
In mismatched conditions without data limitations (Figure \ref{fig:mismatch-maml}), the training set consists of data from all available weekdays.
This captures instances where data collection is comparatively less arduous, although the traffic pattern retains the potential to shift swiftly.
The results are shown in Figure~\ref{fig:mismatch-all}, where each curve in the first three columns denotes the average test prediction MAE of models (resp. the average test prediction MSE of models in the appendix).
We generate two summary figures illustrating the average and maximum MAE across all six training sets.
% , using training data from different weekday periods.
% MAE versus time using STGCN, MTGNN, GTS, and RDGCN over different time intervals in the training set.
In Figure~\ref{fig:mismatch}, we compare the performance of our approach with the STGCN, MTGNN, GTS, MTGODE, STGNCDE, and RDGCN in the mismatched data,
when the training process is augmented with MAML.
% respectively. 
Figure~\ref{fig:mismatch-maml} plots the 
prediction MAE of baseline models and RDGCN over time, given all available weekday data.
% Note that identical training and test data are used for all the results in Figure~\ref{fig:mismatch} and Figure~\ref{fig:mismatch-maml}.

In Figure~\ref{fig:mismatch}, 
all RDGCN models have nearly identical performance regardless of which time window of data is used for training.
The MAE of all the RDGCN
models is uniformly low (i.e., small y-axis values), and there is very low variance in performance across RDGCN models
trained with different time windows (i.e., the curves of average MAE is close to the curves of maximum MAE).
% Figure~\ref{fig:mismatch} indicates that predictions of the RDGCN model trained on all different training set have low MAE and low variance over time. 
% More importantly, all the RDGCN models have nearly identical performance regardless of which period's data is used for training. 
However, the performance of baseline models are significantly different depending on the training set, and some can have a relatively high MAE (e.g., the curve of STGCN on Pems-bay dataset has much higher MAE values than the one of RDGCN over time).
From Figure~\ref{fig:mismatch-maml}, we can see that even when the model is trained using all available weekday data, RDGCN outperforms
the baseline models wherein the variance across time, and across models is very low. While more data brings some
gain to baseline models, its impact on RDGCN is fairly limited, indicating that RDGCN performs well in different
testing domains without needing additional training data.

These test results support our hypothesis that incorporating traffic dynamics into the learning model makes it more robust to this kind of domain generalization (data from mismatched training and testing conditions).
We speculate that this is a consequence of our model capturing the relative changes in speed through the dynamical equations, whereas existing baseline models are black box models that derive complex functions of the absolute speed values across time. In effect, when there is a mismatch, the underlying nature of traffic dynamics is less likely to be impacted, whereas the complex patterns of absolute speed values might vary significantly across domains. This is particularly true when dealing with limited data that does not contain all possible patterns.
At the same time, RDGCN is designed to predict based on neighboring vertices, so even if the speed patterns of a distant sensor and a close sensor are similar (e.g., both are free flow), the model uses close sensors to make predictions.
We note that the prediction of RDGCN is not uniformly better than the prediction of baselines {(e.g., the prediction of MTGNN trained by Seattle weekday data from 8:00 to 12:00 is better than the prediction of RDGCN)}, and one possible reason is that speed pattern mismatches between weekdays and weekends are not always significant (e.g., when the training weekday is a holiday).
Furthermore, the predictions of MTGNN and MTGODE exhibit a slight superiority over RDGCN in Metr-la dataset in certain windows. Our conjecture is that the mix-hop layers enable these models to assign higher significance to learn short-term patterns, which likely does not change much between the training and test data. Although real-world data under situations such as disasters or events are hard to obtain, our approach of
splitting the dataset emulates test scenarios that are sufficiently different from the training dataset to demonstrate the robustness of our approach.
\begin{table}[htbp]
\centering
\vspace{-0.1in}
  \caption{Evaluation of models under mismatched data}
  \vspace{-0.1in}
\resizebox{0.85\columnwidth}{!}{
    \begin{tabular}{ccccc}
\midrule

 & {Dataset} & ColaGNN & EpiGNN  & SIRGCN \\ \midrule
              % & \multicolumn{3}{c}{MAE}  \\ \cmidrule(lr){2-5}
 \multirow{2}{*}{MAE} & Japan-Prefectures & 356 $\pm$ 21  & 466 $\pm$ 24 & \textbf{342} $\pm$ \textbf{22}  \\
 & US-States & 46 $\pm$ 3 & 66 $\pm$ 6  & \textbf{41} $\pm$ \textbf{4}  \\ \midrule
 % &\multicolumn{3}{c}{RMSE} \\ \cmidrule(lr){2-5}
 \multirow{2}{*}{RMSE} & Japan-Prefectures &  901 $\pm$ 53 & 922 $\pm$ 69 & \textbf{863} $\pm$ \textbf{44} \\
 & US-States & 130 $\pm$ 12 & 178 $\pm$ 16 & \textbf{121} $\pm$ \textbf{10} \\
\midrule
\end{tabular}%
\vspace{-0.5in}
}
\label{tab:sir-metrices}
\end{table}
% \vspace{-0.05in}

% \begin{table}[htbp]
% \small
%   \caption{Evaluation metrics of baseline models and SIRGCN under mismatched data.}
% \resizebox{\columnwidth}{!}{
%     \begin{tabular}{c|c|c|c|c}
% \midrule

% % 
%  & {Dataset} & ColaGNN & EpiGNN  & SIRGCN \\ \midrule
%               % & \multicolumn{3}{c}{MAE}  \\ \cmidrule(lr){2-5}
%  \multirow{2}{*}{MAE} & Japan-Prefectures & 356 $\pm$ 21  & 466 $\pm$ 24 & 342 $\pm$ 22  \\
%  & US-States & 46 $\pm$ 3 & 66 $\pm$ 6  & 41 $\pm$ 4  \\ \midrule
%  % &\multicolumn{3}{c}{RMSE} \\ \cmidrule(lr){2-5}
%  \multirow{2}{*}{RMSE} & Japan-Prefectures &  901 $\pm$ 53 & 922 $\pm$ 69 & 863 $\pm$ 44 \\
%  & US-States & 130 $\pm$ 12 & 178 $\pm$ 16 & 121 $\pm$ 10 \\
% \midrule
% \end{tabular}%
% }

% \label{tab:sir-metrices}
% \end{table}%
% \normalsize

% \subsubsection{SIRGCN}
\textbf{Mismatched Data Experiments for SIRGCN.} We explore the performance of SIRGCN under mismatched situations. 
Since infection spread and travel patterns vary from season to season, we train our model and the baseline models with ILI data recorded in Summer and Winter, and test the predictions on data in Spring and Fall. 
The result is shown in Table~\ref{tab:sir-metrices}, where each element denotes the MAE and MSE under different seasons. 

The results demonstrate that SIRGCN performs consistently well under the mismatched data scenario with low MAE and RMSE compared to the baseline models. Although SIRGCN does not significantly outperform the deep-learning-based ColaGNN model, we note that SIRGCN makes predictions using only the latest observation at $1$ time point augmented by approximating the total susceptible and recovered populations as specified by the domain equations, whereas the baselines which consider the disease propagation as a black-box model, require more than 7 years data to train, and twenty weeks worth data to make their predictions. 

The two datasets are used for testing, but the theory can also apply to other applications, such as air quality forecasting, molecular simulation, and others where there are underlying graphical models and the domain ODE is well developed.
Overall these evaluations validate the
the main hypothesis of this paper where integrating domain differential equations into GCN allows for better
robustness.
% to domain generalization.

\vspace{-0.1in}
\section{Conclusion}
In this paper, we investigate the challenging problem of graph time series prediction when training and test data are drawn from different or mismatched scenarios. 
To address the challenge, we proposed a methodological approach to integrate domain differential equations in graph convolutional networks to capture the common data behavior across data distributions. We theoretically justify the robustness of this approach under certain conditions on the underlying domain and data. 
By operationalizing our approach, we gave rise to two novel domain-informed GCNs: RDGCN and SIRGCN. These architectures fuse traffic speed reaction-diffusion equations, and Susceptible-Infected-Recovered infectious disease spread equations, respectively. Through rigorous numerical evaluation, we demonstrate the robustness of our models in mismatched data scenarios. The findings showcased in this work underscore the transformative potential of domain-ODE-informed models as a burgeoning category within the domain of graph neural networks. This framework paves the way for future exploration in addressing the challenges of domain generalization in other contexts.

\section{Acknowledgements}
This work was partly funded through a Lehigh internal Accelerator Grant, the CCF-1617889 and the IIS-1909879 grants from the National Science Foundation.
Sihong was supported in part by the Education Bureau of Guangzhou Municipality and the Guangzhou-HKUST (GZ) Joint Funding Program (Grant No.2023A03J0008).
Rick S. Blum was supported by the U.S. Office of Naval Research under Grant
N00014-22-1-2626.

% \clearpage
\bibliography{aaai24}

\clearpage
\appendix

\section*{APPENDIX}

\section{Proofs}\label{subsec:proof}

% \begin{customlemma}{2}\label{two}
% \label{lemma:f=h}
% If $h_2$ is trained with data sampled from $\mathcal{X}_s$ such that assumption~\ref{assump-symmetric} is true, and the loss function ${L}$ is L1-norm,  
% then $h_2^\ast = f$.
% \end{customlemma}

\subsection{Proof of Lemma 2}
\begin{customlemma}{2}
If (1) $h_2$ is trained with data sampled from $\mathcal{X}_s$ such that assumption~\ref{assump-symmetric} is true, (2) the loss function ${L}$ is the L1-norm or MSE,  
then $h_2^\ast = F$.
\end{customlemma}

\begin{proof}
We prove this by contradiction. If $h_2^\ast \neq F$, there must exist $\hat{h}_2^\ast(X_t) \neq 0$ such that $h_2^\ast(X_t) = F(O(X(t), \mathcal{A});\Theta_1) + \hat{h}_2^\ast(X_t)$ and $\hat{h}_2^\ast$ minimizes the following expectation:
% of the loss between $\hat{h}_2$ and $g_s$: 
\small
\begin{equation}
\label{eq:lemma2-a.1}
\begin{aligned}
    \hat{h}_2^\ast & = \min_{\hat{h}_2 : h_2 - F} \mathbb{E}_{X_{t-T:t} \sim \mathcal{D}}[{L}(\hat{h}_2(X(t)), \\
    & \qquad \qquad \qquad G_s(O(X(t), \mathbb{I}-\mathcal{A}), X_{t-T:t-1};\Theta_2))].    
\end{aligned}
\end{equation}
\normalsize

If the loss function is the L1-norm, Problem~(\ref{eq:lemma2-a.1}) is minimized when $\hat{h}_2^\ast({X}_t)$ equals the median of $G_s(O(X(t), \mathbb{I}-\mathcal{A}), X_{t-T:t-1};\Theta_2)$. 
Assumption \ref{assump-symmetric} in Section 5 implies
\small
\begin{equation}
\label{eq:assump-symmetric-implies}
    Median_{X_{t-T:t} \sim \mathcal{D}} [G] = \mathbb{E}_{X_{t-T:t} \sim \mathcal{D}} [U]=0.
\end{equation}
\normalsize
% By Eq (\ref{eq:assump-symmetric-implies}) in assumption \ref{assump-symmetric}, 
Thus $\hat{h}_2^\ast({X_t})=0$ is the optimal solution of Problem~(\ref{eq:lemma2-a.1}),
% Hence $\hat{h}_2({X_t})=0$ must minimize expectation of loss 
% $\mathbb{E}_{X_{t-T:t} \sim \mathcal{D}}[{L}(f(\{x_{t,i}, x_{t, j} | j \in \mathcal{N}_i \}) + \hat{h}_2(X(t)), f(\{x_{t,i}, x_{t, j} | j \in \mathcal{N}_i \}) + G_s(O(X(t), \mathbb{I}-\mathcal{A}), X_{t-T:t-1};\Theta_2)]$, 
% $\mathcal{L}_{(\mathcal{D}, f+g_s)}(f+\hat{h}_2)$, 
which contradicts the fact that $\hat{h}_2^\ast({X_t}) \neq 0$. \\
If the loss function is the MSE, 
% we also set up a contradiction by assuming that if $h_2^\ast \neq F$, 
there must exist $\hat{h}_2(X_t) \neq 0$ such that $h_2^\ast(X_t) = F(O(X(t), \mathcal{A});\Theta_1)$ and $\hat{h}_2(X_t)$ minimize the following expectation
\small
\begin{equation}
\label{eq:contradict-mse}
    \mathbb{E}_{X_{t-T:t} \sim \mathcal{D}} [(\hat{h}_2(X_t) - G_s(O(X(t), \mathbb{I}-\mathcal{A}), X_{t-T:t-1};\Theta_2))^2].
\end{equation}
\normalsize
Since Assumption \ref{assump-symmetric} in Section 5 implies 
\small
\begin{equation}
    \mathbb{E}_{X_{t-T:t} \sim \mathcal{D}}[G_s(O(X(t), \mathbb{I}-\mathcal{A}), X_{t-T:t-1};\Theta_2)] = 0,
\end{equation}
\normalsize
the expectation $\mathbb{E}_{X_{t-T:t} \sim \mathcal{D}} [(\hat{h}_2(X_t) - G_s(O(X(t), \mathbb{I}-\mathcal{A}), X_{t-T:t-1};\Theta_2))^2]$ is minimized when the derivative $2(\hat{h}_2(X(t)) - \mathbb{E}_{X_{t-T:t} \sim \mathcal{D}} [G_s(O(X(t), \mathbb{I}-\mathcal{A}), X_{t-T:t-1};\Theta_2)]) = 0$, hence $\hat{h}_2({X_t})=0$ must minimize the expectation in Eq. (\ref{eq:contradict-mse}), which contradicts the fact that $\hat{h}_2({X_t}) \neq 0$. Therefore $h_2^\ast = F$.
\end{proof}
% \begin{customthm}{3}\label{three}
% \label{theorem:robust}
% If (1) the training data is sampled from the source domain where assumption \ref{assump-symmetric} is true, (2) the loss function ${L}{(h, l)}$ obeys the triangle inequality, 
% % (3) $\epsilon_1, \epsilon_2 \leq \min(\mathbb{E}_{X_{t-T:t} \sim \mathcal{D}}[|G_s(O(X(t), \mathbb{I}-\mathcal{A}), X_{t-T:t-1};\Theta_2))|], \mathbb{E}_{X_{t-T:t} \sim \mathcal{D}}[|G_\tau(O(X(t), \mathbb{I}-\mathcal{A}), X_{t-T:t-1};\Theta_2)|])$, 
% then 
% % with probability of at least $1-\delta$,
% \begin{equation}
%     disc(\mathcal{H}_2^\ast) \leq disc(\mathcal{H}_1^\ast).
% \end{equation}
% \end{customthm} 
\subsection{Proof of Theorem 3}
\begin{customthm}{3}
If (1) the training data is sampled from the source domain where assumption \ref{assump-symmetric} is true, (2) the loss function ${L}{(h, l)}$ obeys the triangular equality, 
% (3) $\epsilon_1, \epsilon_2 \leq \min(\mathbb{E}_{X_{t-T:t} \sim \mathcal{D}}[|g_s(\{x_{t,i}, x_{t, j} | j \notin \mathcal{N}_i \}))|], \mathbb{E}_{X_{t-T:t} \sim \mathcal{D}}[|g_\tau(\{x_{t,i}, x_{t, j} | j \notin \mathcal{N}_i \})|])$, 
then the discrepancy with any triangular equality loss should satisfy
% with probability of at least $1-\delta$,
\begin{equation}
    disc(\mathcal{H}_2^\ast) \leq disc(\mathcal{H}_1^\ast).
\end{equation}    
\end{customthm}

\begin{proof}
By the definition of discrepancy in Eq.(\ref{eq:disc}), we know
\small
\begin{equation}
\label{eq:imperfect-mae-h1-1}
\begin{aligned}
     disc(\mathcal{H}_1^\ast) &= \sup_{h_1\in \mathcal{H}_1^\ast} |\mathcal{L}_{(\mathcal{D}, F+G_s)}(h_1) - \mathcal{L}_{(\mathcal{D}, F+G_\tau)}(h_1)|   \\
    &= \sup_{h_1\in \mathcal{H}_1^\ast} 
    |\mathbb{E}_{X_{t-T:t} \sim \mathcal{D}} [L(h_1(X_{t-T:t}), F(X(t)) \\
    & \qquad \qquad \qquad + G_s(O(X(t), \mathbb{I}-\mathcal{A}), X_{t-T:t-1};\Theta_2)) \\
    & \qquad \qquad \qquad - L(h_1(X_{t-T:t}), F(X(t)) \\
    & \qquad \qquad \qquad + G_\tau(O(X(t), \mathbb{I}-\mathcal{A}), X_{t-T:t-1};\Theta_2))]| \\
    &\stackrel{\rm (a)}{\leq} \sup_{h_1\in \mathcal{H}_1^\ast} \mathbb{E}_{X_{t-T:t} \sim \mathcal{D}} [|L(h_1(X_{t-T:t}), F(X(t)) \\
    & \quad + G_s(O(X(t), \mathbb{I}-\mathcal{A}), X_{t-T:t-1};\Theta_2)) \\ 
    & \qquad \qquad \qquad - L(h_1(X_{t-T:t}),  F(X(t)) \\
    & \qquad \qquad \qquad + G_\tau(O(X(t), \mathbb{I}-\mathcal{A}), X_{t-T:t-1};\Theta_2))|] \\
    &\stackrel{\rm (b)}{\leq} \mathbb{E}_{X_{t-T:t} \sim \mathcal{D}} [L(G_s(O(X(t), \mathbb{I}-\mathcal{A}), X_{t-T:t-1};\Theta_2), \\
    & \qquad \qquad \qquad \qquad G_\tau(O(X(t), \mathbb{I}-\mathcal{A}), X_{t-T:t-1};\Theta_2))],
\end{aligned}
\end{equation}
\normalsize
where (a) follows from Jensen's equality ($|\cdot|$ is convex) and (b) follows from the triangle inequality (which implies $|L(x, y)| \geq |L(x, z) - L(y, z)|$, for any $x,y,z \in \mathbb{R}$).
By Assumption~\ref{assump:realizability} in Section 5, we can set $h_1^\ast = F+G_s$ where $\mathcal{L}_{(\mathcal{D}, F+G_s)} (h_1^\ast) = 0$. Then the discrepancy of $\mathcal{H}_1$ is
% By the definition of supremum, let $\mathcal{L}_{(\mathcal{D}, F+G_s)} (h_1^\ast) = 0$,
\small
\begin{equation}
\label{eq:imperfect-mae-h1-2}
\begin{aligned}
    disc(\mathcal{H}_1^\ast) \stackrel{\rm (c)}{\geq} & \mathbb{E}_{X_{t-T:t} \sim \mathcal{D}} [L(F(X(t)) \\
    & \qquad + G_\tau(O(X(t), \mathbb{I}-\mathcal{A}), X_{t-T:t-1};\Theta_2), F(X(t)) \\
    & \qquad + G_s(O(X(t), \mathbb{I}-\mathcal{A}), X_{t-T:t-1};\Theta_2))] \\
    = & \mathbb{E}_{X_{t-T:t} \sim \mathcal{D}} [L(G_s(O(X(t), \mathbb{I}-\mathcal{A}), X_{t-T:t-1};\Theta_2), \\
    & \qquad \qquad \qquad G_\tau(O(X(t), \mathbb{I}-\mathcal{A}), X_{t-T:t-1};\Theta_2))],    
\end{aligned}
\end{equation}
\normalsize
where (c) follows from the definition that the supremum (the least element that is greater than or equal to each element in the set).
Thus from (\ref{eq:imperfect-mae-h1-1}) and (\ref{eq:imperfect-mae-h1-2}) together 
\small
\begin{equation}
\begin{aligned}
    disc(\mathcal{H}_1^\ast) & = \mathbb{E}_{X_{t-T:t} \sim \mathcal{D}} [L(G_s(O(X(t), \mathbb{I}-\mathcal{A}), X_{t-T:t-1};\Theta_2), \\
    & \qquad \qquad \qquad G_\tau(O(X(t), \mathbb{I}-\mathcal{A}), X_{t-T:t-1};\Theta_2))].    
\end{aligned}
\end{equation}
\normalsize
For $\mathcal{H}_2$, by the triangle inequality,
\small
\begin{equation}
\label{eq:imperfect-mae-h1-a}
\begin{aligned}
    disc(\mathcal{H}_2^\ast) &= \sup_{h_2\in \mathcal{H}_2^\ast} |\mathcal{L}_{(\mathcal{D}, F+G_s)}(h_2) - \mathcal{L}_{(\mathcal{D}, F+G_\tau)}(h_2)|   \\
    &\leq \mathbb{E}_{X_{t-T:t} \sim \mathcal{D}} [L(G_s(O(X(t), \mathbb{I}-\mathcal{A}), X_{t-T:t-1};\Theta_2), \\
    & \qquad \qquad \qquad G_\tau(O(X(t), \mathbb{I}-\mathcal{A}), X_{t-T:t-1};\Theta_2))].
\end{aligned}
\end{equation}
\normalsize
Hence we have shown that $disc(\mathcal{H}_2^\ast) \leq disc(\mathcal{H}_1^\ast)$. 
% The equality $disc(\mathcal{H}_2) = disc(\mathcal{H}_1)$ holds if $\epsilon_2 \geq \min(|\mathbb{E}_{X_{t-T:t} \sim \mathcal{D}}[G_s(O(X(t), \mathbb{I}-\mathcal{A}), X_{t-T:t-1};\Theta_2))], \mathbb{E}_{X_{t-T:t} \sim \mathcal{D}}[G_\tau(O(X(t), \mathbb{I}-\mathcal{A}), X_{t-T:t-1};\Theta_2)]|)$.
\end{proof}

\subsection{Discrepancy using MSE}

\begin{assumption}
\label{assump-negativecoor}
Let $U'=G_s(O(X(t), \mathbb{I}-\mathcal{A}), X_{t-T:t-1};\Theta_2)  G_\tau(O(X(t), \mathbb{I}-\mathcal{A}), X_{t-T:t-1};\Theta_2)$ be a random variable where ${X_{t-T:t} \sim \mathcal{D}}$,
$\mathbb{E}_{{X_{t-T:t} \sim \mathcal{D}}}[U'] \leq 0$.
     
    %  i.e., 
    % \begin{equation}
    % \label{eq:assump-symmetric}
    %     \mathbb{P}[g_s(\{x_{t,i}, x_{t, j} | j \notin \mathcal{N}_i \})] = \mathbb{P}[-g_s(\{x_{t,i}, x_{t, j} | j \notin \mathcal{N}_i \})]. 
    % \end{equation}
\end{assumption}

\begin{corollary}
\label{theorem:robust-mse}
If (1) the training data is sampled from the source domain where assumption \ref{assump-symmetric} is true, 
(2) the labeling function in the source and target domain satisfy Assumption~\ref{assump-negativecoor}, 
(3) the loss function ${L}{(h, l)}$ is MSE, 
(4) $\mathcal{L}_{(\mathcal{D}, F)}(h_2^\ast) = 0$, 
then
% \begin{equation}
% \begin{aligned}
%     {\epsilon_2}' - \epsilon_1' \leq & \frac{1}{2}\mathbb{E}_{X_{t-T:t} \sim \mathcal{D}}[|G_s(O(X(t), \mathbb{I}-\mathcal{A}), X_{t-T:t-1};\Theta_2)- G_\tau(O(X(t), \mathbb{I}-\mathcal{A}), X_{t-T:t-1};\Theta_2)|- \\
%     &|G_s(O(X(t), \mathbb{I}-\mathcal{A}), X_{t-T:t-1};\Theta_2) +
%     G_\tau(O(X(t), \mathbb{I}-\mathcal{A}), X_{t-T:t-1};\Theta_2)|],  
% \end{aligned}
% \end{equation}
% (3) $\epsilon_1, \epsilon_2 \leq \min(\mathbb{E}_{X_{t-T:t} \sim \mathcal{D}}[|g_s(\{x_{t,i}, x_{t, j} | j \notin \mathcal{N}_i \}))|], \mathbb{E}_{X_{t-T:t} \sim \mathcal{D}}[|g_\tau(\{x_{t,i}, x_{t, j} | j \notin \mathcal{N}_i \})|])$, 
% then there exists $\epsilon_1'$, $\epsilon_2'$ such that if the training loss $\mathcal{L}_{\mathcal{D}, F+G_s} (h_1^*) < \epsilon_1'$ and $\mathcal{L}_{\mathcal{D}, F+G_\tau} (h_2^*) < \epsilon_2'$,
% with probability of at least $1-\delta$,
then $disc(\mathcal{H}_2^\ast) \leq disc(\mathcal{H}_1^\ast)$.
% \begin{equation}
%     disc(\mathcal{H}_2^\ast) \leq disc(\mathcal{H}_1^\ast).
% \end{equation}
\end{corollary} 

\begin{proof}

% Let $\Delta_1(X_{t-T:t}) = h_1^\ast - F(O(X(t), \mathcal{A});\Theta_1) - G_s(O(X(t), \mathbb{I}-\mathcal{A}), X_{t-T:t-1};\Theta_2)$ denote the difference $h_1^*$ learns, and 
% $\Delta_2(X(t)) = h_2^\ast - F(O(X(t), \mathcal{A});\Theta_1)$.
By Assumption~\ref{assump:realizability}, we can set $h_1^\ast = F+G_s$ where $\mathcal{L}_{(\mathcal{D}, F+G_s)} (h_1^\ast) = 0$, then the discrepancy of $\mathcal{H}_1$ is
% By the definition of supremum, let $\mathcal{L}_{(\mathcal{D}, F+G_s)} (h_1^\ast) = 0$,
\small
\begin{equation}
\label{eq:imperfect-mse-h1-2}
\begin{aligned}
    & disc (\mathcal{H}_1^\ast) = \sup_{h_1\in \mathcal{H}_1^\ast} |\mathbb{E}_{X_{t-T:t} \sim \mathcal{D}} [L(h_1(X_{t-T:t}), F(X(t)) \\
    & \qquad + G_s(O(X(t), \mathbb{I}-\mathcal{A}), X_{t-T:t-1};\Theta_2)) \\
    & \qquad -  L(h_1(X_{t-T:t}),  F(X(t)) \\
    & + G_\tau(O(X(t), \mathbb{I}-\mathcal{A}), X_{t-T:t-1};\Theta_2))]| \\
    \stackrel{\rm (d)}{\geq} & \mathbb{E}_{X_{t-T:t} \sim \mathcal{D}} [L(F(X(t)) + G_\tau(O(X(t), \mathbb{I}-\mathcal{A}), X_{t-T:t-1};\Theta_2), \\ 
    & \qquad \qquad \qquad F(X(t)) + G_s(O(X(t), \mathbb{I}-\mathcal{A}), X_{t-T:t-1};\Theta_2))] \\
    = & \mathbb{E}_{X_{t-T:t} \sim \mathcal{D}} [(G_s(O(X(t), \mathbb{I}-\mathcal{A}), X_{t-T:t-1};\Theta_2) \\
    & - G_\tau(O(X(t), \mathbb{I}-\mathcal{A}), X_{t-T:t-1};\Theta_2))^2],    
\end{aligned}
\end{equation}
\normalsize
where (d) follows from the definition of the supremum (the least element that is greater than or equal to each element in the set).
We note that the triangular equality is not necessarily true in this case thus we cannot find the upper bound of $disc(\mathcal{H}_1^\ast)$.

Since $\mathcal{L}_{(\mathcal{D}, F)}(h_2^\ast) = \mathbb{E}_{X_{t-T:t} \sim \mathcal{D}} [((h_2^\ast(X(t)) - F(O(X(t), \mathcal{A});\Theta_1))^2] = 0$ implies that $h_2^\ast = F(O(X(t), \mathcal{A});\Theta_1)$. 

Then the discrepancy of $\mathcal{H}_2^\ast$ is:
\small
\begin{equation}
\label{eq:imperfect-mse-h2-2}
\begin{aligned}
    disc (\mathcal{H}_2^\ast) & = \sup_{h_2^\ast\in \mathcal{H}_2^\ast}|\mathbb{E}_{X_{t-T:t} \sim \mathcal{D}} [(h_2^\ast(X(t)) - F(O(X(t), \mathcal{A});\Theta_1) \\
    & - G_s(O(X(t), \mathbb{I}-\mathcal{A}), X_{t-T:t-1};\Theta_2))^2  \\
    & - (h_2^\ast(X(t)) - F(O(X(t), \mathcal{A});\Theta_1) \\
    & - G_\tau(O(X(t), \mathbb{I}-\mathcal{A}), X_{t-T:t-1};\Theta_2))^2]| \\
    & \stackrel{\rm (e)}{\leq} \sup_{h_2^\ast \in \mathcal{H}_2^\ast}\mathbb{E}_{X_{t-T:t} \sim \mathcal{D}} [|(h_2^\ast(X(t)) - F(O(X(t), \mathcal{A});\Theta_1) \\
    & - G_s(O(X(t), \mathbb{I}-\mathcal{A}), X_{t-T:t-1};\Theta_2))^2 - \\
    & (h_2^\ast(X(t)) - F(O(X(t), \mathcal{A});\Theta_1) \\
    & - G_\tau(O_2(H_{t,T},
    \mathcal{G});\Theta_2))^2|] \\
    & = \mathbb{E}_{X_{t-T:t} \sim \mathcal{D}}[|(G_s(O(X(t), \mathbb{I}-\mathcal{A}), X_{t-T:t-1};\Theta_2))^2 \\
    & -(G_\tau(O(X(t), \mathbb{I}-\mathcal{A}), X_{t-T:t-1};\Theta_2))^2|] \\
    & \leq \mathbb{E}_{X_{t-T:t} \sim \mathcal{D}}[(G_s(O(X(t), \mathbb{I}-\mathcal{A}), X_{t-T:t-1};\Theta_2))^2 \\
    & +(G_\tau(O(X(t), \mathbb{I}-\mathcal{A}), X_{t-T:t-1};\Theta_2))^2] \\
    & \stackrel{\rm (f)}{\leq} \mathbb{E}_{X_{t-T:t} \sim \mathcal{D}}[|(G_s(O(X(t), \mathbb{I}-\mathcal{A}), X_{t-T:t-1};\Theta_2)\\
    &-G_\tau(O(X(t), \mathbb{I}-\mathcal{A}), X_{t-T:t-1};\Theta_2))^2|],
\end{aligned}
\end{equation}
\normalsize
where (e) follows from Jensen's inequality,
(f) follows from Assumption~\ref{assump-negativecoor}.
Hence by Eq. (\ref{eq:imperfect-mse-h1-2})(\ref{eq:imperfect-mse-h2-2}), we know $disc(\mathcal{H}_2^\ast) \leq disc(\mathcal{H}_1^\ast)$
\end{proof}

% \begin{figure}
%     \centering
%     \includegraphics[width=0.99\columnwidth]{figs/ablation/plot-gs-covid.pdf}
%     % \vspace{-.2in}
%     \caption{The probability density function of the random variable, $G$ is symmetric about 0 for all seasons. We randomly select 3 vertices in each data set.}
%     \label{fig:symmetric-covid}
% % \vspace{-.2in}
% \end{figure}

% \lipsum[1-2]

\section{Data Support}\label{sec:data-support}
\begin{figure}[h!]
% \vspace{-.28in}
    % \flushup
    \subfloat[Traffic speed prediction.]
    {\hspace{-.0in}\includegraphics[width=0.95\columnwidth]{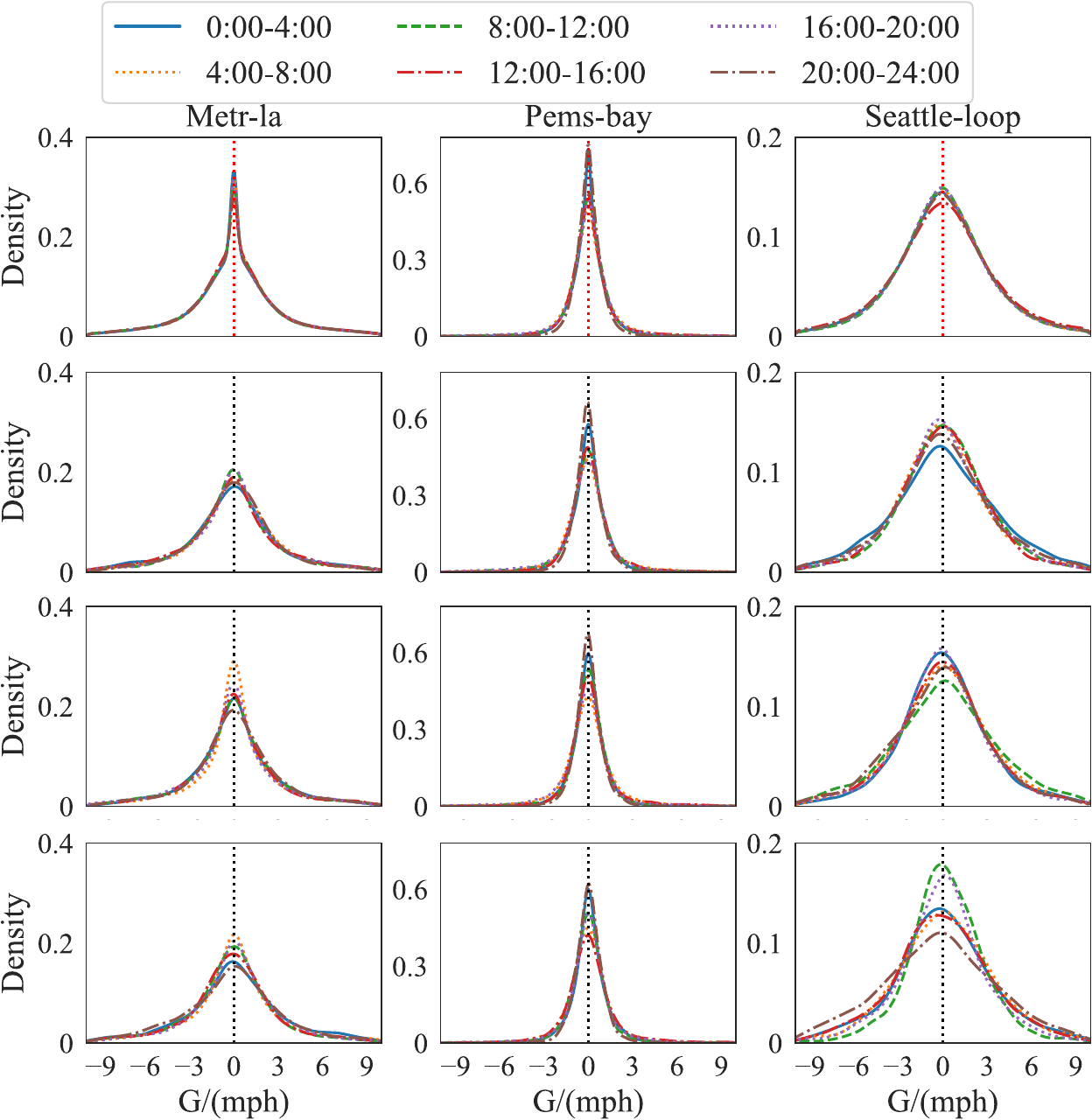}
    \label{fig:symmetric}}
    \quad
    \subfloat[ILI prediction.]{
    \vspace{-.0in}\includegraphics[width=0.95\columnwidth]{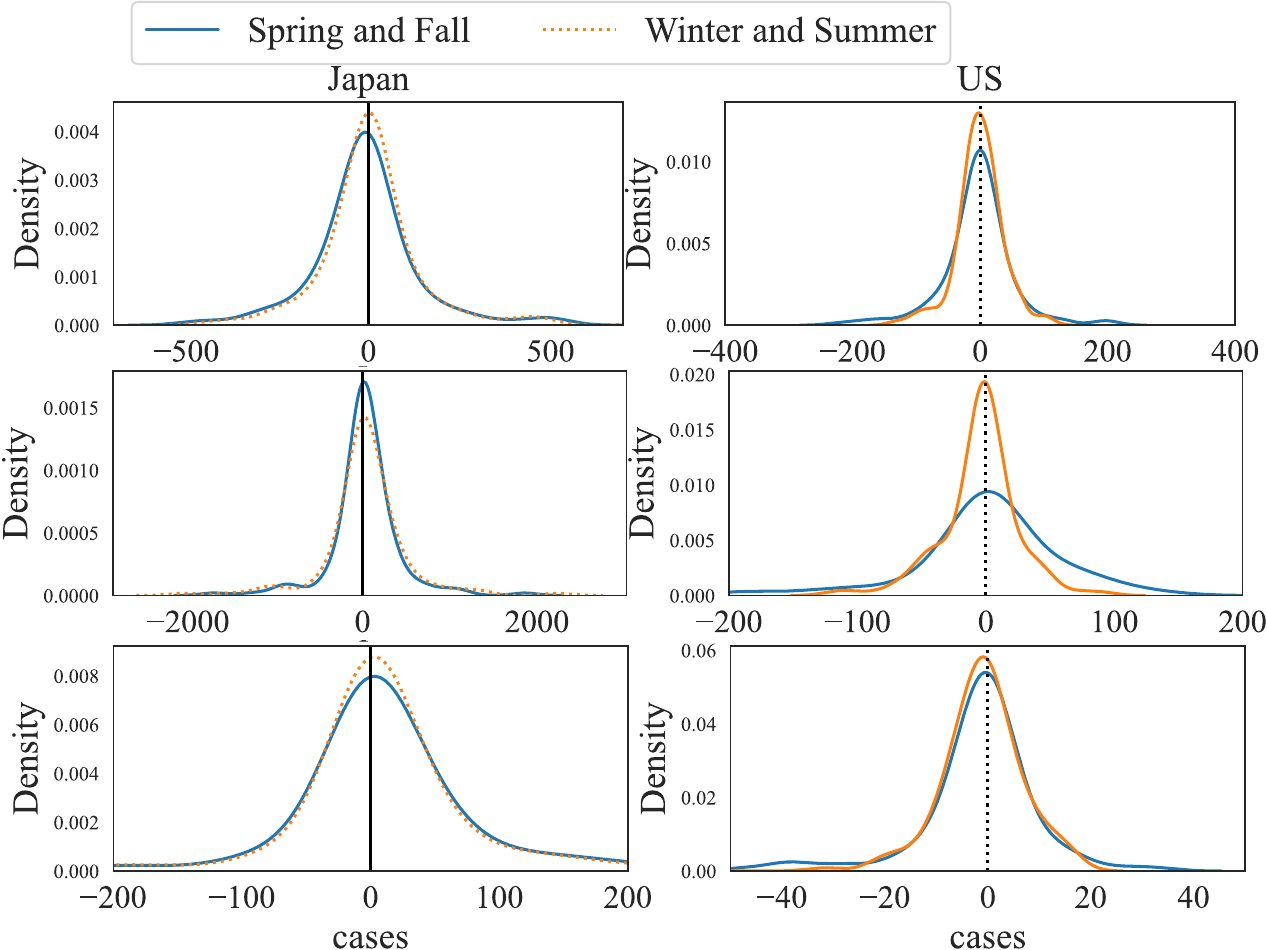}
    \label{fig:symmetric-covid}}
    \caption{(a) The pdf of the random variable, $G$ is symmetric about 0 for all the time periods.
    Figures in the first row are the mixed distribution of all sensors. 
    Figures in the following three rows are the distribution of three randomly selected sensors in each dataset. 
    (b) The pdf of the random variable, $G$ is symmetric about 0 for all seasons. We randomly select 3 vertices in each data set.}
    % \vspace{6.98in}
\end{figure}

% We verify assumption \ref{assump-symmetric} by assuming RDGCN approximates function $f$ perfectly, then
We empirically verify the condition in Assumption \ref{assump-symmetric}, in the scenario that the Reaction-Diffusion traffic model is the underlying
physical law, and consequently, RDGCN, when trained well, perfectly models function $f$. Then,
\small
\begin{equation}
\begin{aligned}
   & G_s(O(X(t), \mathbb{I}-\mathcal{A}), X_{t-T:t-1};\Theta_2) \\
   = & X_{t+1} - F(O(X(t), \mathcal{A});\Theta_1)) \\
 \approx & X_{t+1} - RDGCN(X(t)).    
\end{aligned}
\end{equation}
\normalsize

% \vspace{-1.0in}

The probability density function (pdf) of the random variable $G=G_s(O(X(t), \mathbb{I}-\mathcal{A}), X_{t-T:t-1};\Theta_2)$ in all the six periods
 are shown in Figure~\ref{fig:symmetric}, 
% where the x-axis is the value of $G$ and the y-axis is the corresponding density of the value. 
which plots the empirical density of the variable $G$ in each dataset.
% The data used is the weekday data of all vertices in the corresponding time period excluding the missing data.
% The figure is generated using the weekday data of all vertices in the corresponding time period excluding the missing data.
As can be observed, the empirical density adheres to the condition in Assumption~\ref{assump-symmetric}.
For SIRGCN, we approximate $G_s$ by

\begin{equation}
\begin{aligned}
        & G_s(O(X(t), \mathbb{I}-\mathcal{A}), X_{t-T:t-1};\Theta_2) \\
        = & X_{t+1} - F(O(X(t), \mathcal{A});\Theta_1)) \\
        \approx & X_{t+1} - SIRGCN(X(t)).    
\end{aligned}
\end{equation}
The corresponding pdf plot of the random variable $G$ in ILI prediction is in Figure~\ref{fig:symmetric-covid}.

% \vspace{-0.1in}
\section{Most Important Sensors under Mismatches}\label{sec:important-sensors}

% The integration of the reaction-diffusion model into a computational architecture, as demonstrated by the imputation results in the previous subsection, has a demonstrable impact in connecting the learned model to the underlying physical dynamics of traffic flow. 

% In this section, we explore the reason for the
% baselines' failure in predictions under mismatched data.
In this section, we provide a motivation for going beyond domain-agnostic deep learning models by illustrating a possible weakness of such a model under mismatched data.
Specifically, we apply a post-hoc explanation tool GNNExplainer~\cite{ying2019gnnexplainer} to 
identify the most influential sensors contributing to a model's prediction at the target sensor.
We choose the Spatio-Temporal GCN (STGCN) model which has a good performance in graph time series prediction, 
particularly in traffic speed prediction.
The STGCN model is trained by four-hour data in a sequence of 12 consecutive weekdays, while the GNNExplainer is used to identify the 3 most influential sensors on the weekend data.
We show the location of the 3 most influential sensors under matched data (train by weekday data and test by weekday data), and mismatched data (trained by weekday data and test by weekend data) in Figure~\ref{fig:GNNEXplainer-map}.

Figure~\ref{fig:GNNEXplainer-map} shows that when the test distribution is mismatched with the training distribution, the most influential sensors identified by GNNExplainer are too far to drive within the prediction window, and the distances change significantly. 
In other words, speed measurements from vertices that are too far to influence the target vertex, and suggests a violation of domain traffic law. 
This forms the motivation for our approach.

% The above unreasonable post-hoc explanation may lead to prediction failure under mismatched data,
% while our approach overcomes the challenge by incorporating a traffic law.
% Specifically, the Mean and STD of distances of RDGCN are $0.7 \pm 0.3$, $0.6 \pm 0.5$, $0.5 \pm 0.3$ miles on Metr-la, Pems-bay, Seattle-loop, respectively.
% which implies that the target sensor and the corresponding $3$ most influential sensors are geographically close. 
% As expected, the most influential sensors for the RDGCN are geographically close to the target sensor.

% In Table~\ref{table:average-distance}, note that the top 3 important vertices for STGCN, MTGNN and GTS on the mismatched test set change significantly with different time intervals in the training set, thus indicating that when we have mismatched patterns in the test set, the trained deep learning models cannot identify a vertex geographically close enough to match the target vertices pattern, and as a result, the “data-dependency” identified by GNNExplainer does not correspond to “physical-dependency”. 
% On the other hand, the top 3 vertices for RDGCN remain unchanged, 
% and are always geographically close to
% the target vertex.

\section{Reaction-diffusion Equation}\label{sec:rdgcn}

As seen in Eq.~(\ref{eq:domain-knowledge}), the change in speed is a function of two terms. The diffusion term is a monotone linear function of speed change in the direction of traffic, 
% produces a positive gradient corresponding to a positive speed change in the direction of traffic, 
and it relies on the empirical fact that in the event of congestion, drivers prefer to bypass the congestion by following one of the
neighboring links~(Figure~\ref{fig:physics-diffusion}). 
The reaction term is a non-linear monotone function ($\tanh$ activation)
of speed change opposite to the direction of traffic, and it relies on the empirical fact that a road surrounded by congested roads is highly likely to be congested as well~(Figure~\ref{fig:physics-reaction}).
The architecture of RDGCN is shown in Figure~\ref{fig:architecture}.

\begin{figure}[!ht]
\centering
\includegraphics[width=0.99\columnwidth]{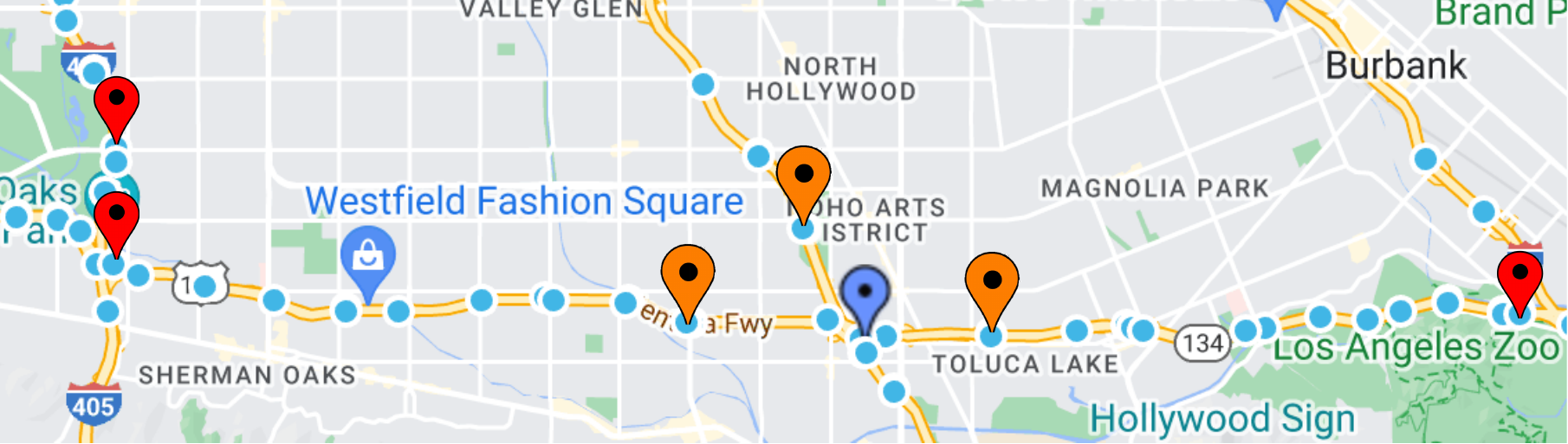} % Reduce the figure size so that it is slightly narrower than the column. Don't use precise values for figure width.This setup will avoid overfull boxes.
\caption{When an STGCN is tested on dataset from a matching distribution, the most important sensors (orange markers) are near the target sesnor, whereas the most important sensors under mismatched data (red markers) for the traffic speed prediction at target sensor (blue marker) are located far away. 
However, under matched data,
the most important sensors are often close to the target sensor.}
\label{fig:GNNEXplainer-map}
% \vspace{-.2in}
\end{figure}

\begin{figure}[!ht]
    \centering
    % \vspace{-.40in}
    \subfloat[Diffusion.]
    {\label{fig:physics-diffusion}
    \hspace{-.0in}\includegraphics[width=0.49\columnwidth]{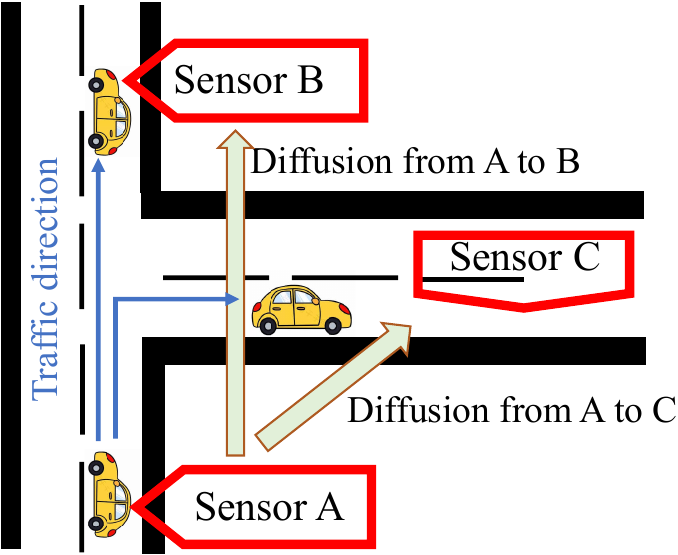}}
    %\caption{fig1}
    % \vspace{-.0in}
    \hfill%
    \subfloat[Reaction.]
    {\label{fig:physics-reaction}
    \hspace{-.11in}\includegraphics[width=0.49\columnwidth]{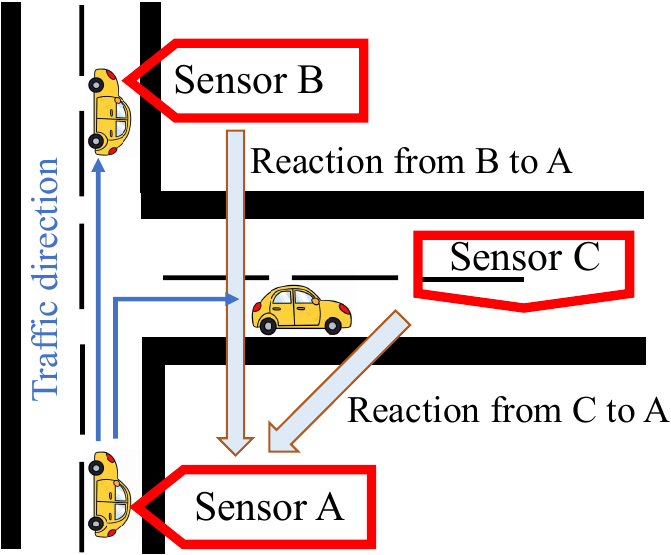}}
    \caption{(a) Diffusion occurs in the direction of a road segment;
    %(for instance sensor A serves as a vertex in the direction of road segments, and congestion at sensor A will diffuse to sensors B and C). 
    (b) reaction occurs opposite to the direction of a road segment. 
    % Congestion at sensors B and C (which are the neighbors of A in opposite road direction) oppositely impacts traffic at sensor A. 
    }
    \label{fig:motivation-1}
    \vspace{-0.1in}
\end{figure}

\begin{figure}[!ht]
\vspace{-.10in}
\centering
\includegraphics[width=0.99\columnwidth]{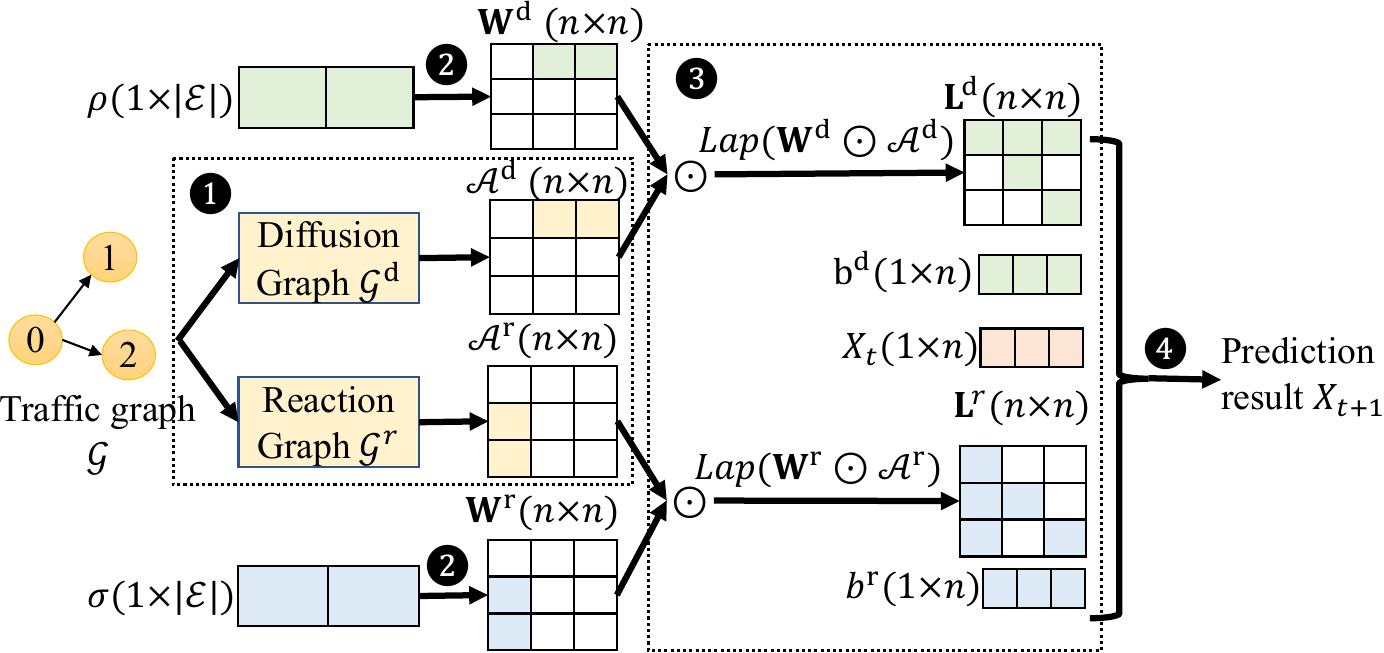} % Reduce the figure size so that it is slightly narrower than the column. Don't use precise values for figure width.This setup will avoid overfull boxes.
\caption{Reaction-diffusion GCN architecture for graph with $|\mathcal{V}|=3$ and $|\mathcal{E}|=2$. $\protect\circled{1}$ derives the diffusion and reaction adjacency matrices $A^d$ and $A^r$; $\protect\circled{2}$ defines model weights $\rho$ and $\sigma$ for the reaction and diffusion
networks, and map them to $\mathbf{W}^d$ and $\mathbf{W}^r$ with weights $\rho$ and $\sigma$; $\protect\circled{3}$ characterizes the Graph Laplacian $\mathbf{L}^d$ and $\mathbf{L}^r$; $\protect\circled{4}$ defines the network prediction function Eq.~(\ref{eq:network-prediction}).}
\label{fig:architecture}
% \vspace{-.18in}
\end{figure}

\section{Ablation Study}\label{sec:ablation}

\subsection{Analysis of RDGCN in Traffic Speed Prediction}

\textbf{Are reaction and diffusion processes essential?}
In this section, we investigate the prediction models incorporating the reaction equation and the diffusion equation, independently, under limited and mismatched data, to understand whether both the reaction and diffusion processes are essential.
We use the same training set (i.e., 12 consecutive working days selected randomly) and test set (i.e., hourly weekend data) in Section~\ref{sec:result}.
The curves of MAE versus time using the model incorporating the reaction equation, the diffusion equation, and the reaction-diffusion equation are shown in Figure~\ref{fig:ablation}.

Figure~\ref{fig:ablation} indicates that the predictions of all models with the reaction-diffusion equation provide low MAE with low variance (i.e., the difference between curves with the highest MAE and lowest MAE is small) over time. However, the predictions of the reaction models only and the diffusion models only have weaker performance in at least one time period. 
We speculate that using only the reaction equation or the diffusion equation is not sufficient to capture the dynamics of the traffic speed change completely. 
Furthermore, the prediction of the model incorporating the reaction-diffusion equation is not uniformly better than the prediction of the model incorporating only the reaction or diffusion equation. 
One possible reason is that the reaction or diffusion process does not always exist in a specific period 
% (e.g., if the road segment is located between a congested segment and an uncongested segment, there may be only a reaction or diffusion process).
(e.g., if two neighboring road segments are in free flow during the test period, the traffic speeds at the two segments do not affect each other. 
Thus there is neither diffusion nor reaction between these two road segments).
% Or in a road where traffic is always flowing smoothly, there may be neither reaction nor diffusion for traffic speed.
These observations further strengthen that both the reaction and diffusion processes are necessary for a reliable prediction.

\textbf{What is the performance under RMSE?}
\begin{figure*}[!h]

\centering
\includegraphics[width=0.9\textwidth]{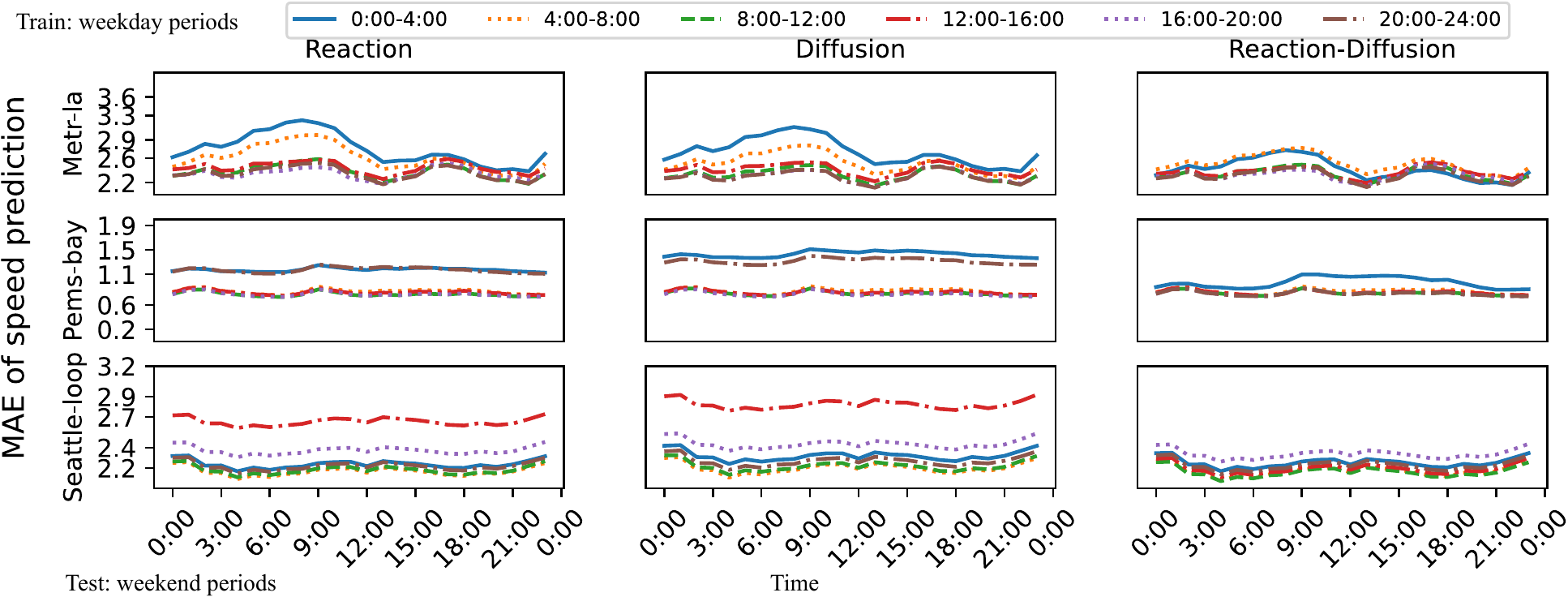} 
% Reduce the figure size so that it is slightly narrower than the column. Don't use precise values for figure width.This setup will avoid overfull boxes.

\caption{
MAE of speed predictions on models incorporating reaction equation, diffusion equation, and reaction-diffusion equation. 
% Models incorporating the reaction-diffusion equation have predictions that are closer to the time period of the training set than models incorporating only the reaction or diffusion equation. 
% The figures show the averaged MAE over five repeated experiments, where models are trained with different 12 continuous weekdays but tested with the same test set sampled on weekends.
% The value of each point is the average of the predictions given by the model, which is trained with the data in the same required time intervals but on five different weekday dataset. 
}

\label{fig:ablation}
\end{figure*}
We plot the corresponding result under RMSE loss in Figure~\ref{fig:mismatch-mse-all}, and the conclusion is consistent with the result using MAE. The RMSE of RDGCN are with low variance regardless of the peroid of the training set.
We acknowledge that RDGCN is not always better than baselines under RMSE, for example, when STGCN is trained with weekday data from 16:00-20:00 in Metr-la.
One possible reason is that the mismatches between the training data and test data are not significant during the corresponding time period. 
The prediction results of RDGCN in terms of RMSE may not always be stable.
For instance, when considering the models for the 4:00 to 8:00 time period in Metr-la, we observe distinct prediction outcomes. This variation could be due to the difference between the pattern of the morning rush hour during selected weekdays and the pattern during weekends.
When the training set includes all available weekday data, the predictions of RDGCN demonstrate stability.

\begin{figure*} 
    \centering
    % \vspace{-.28in}

    \subfloat[Baseline models and RDGCN trained on $12$ consecutive weekdays and augmented by MAML.]
    {\label{fig:mismatch-mae-app}
    \includegraphics[width=0.99\textwidth, height=0.266\textwidth]{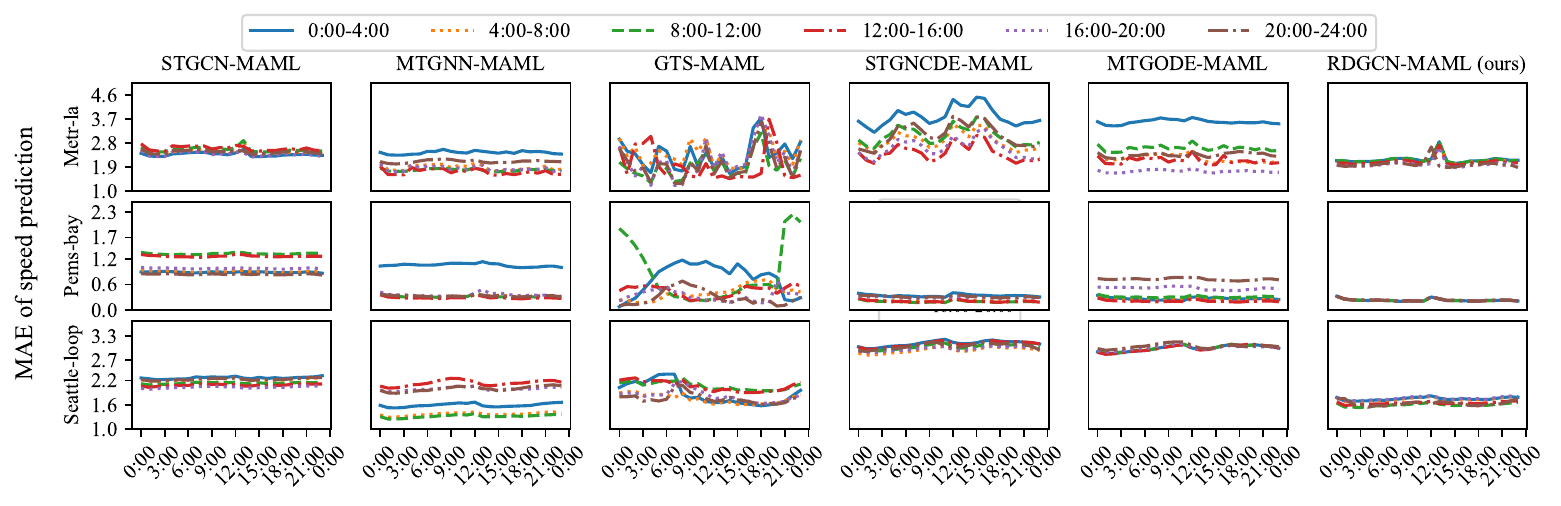}}

    \subfloat[Baseline models and RDGCN trained on $12$ consecutive weekdays and augmented by MAML.]
    {\label{fig:mismatch-mse}
    \includegraphics[width=0.99\textwidth, height=0.266\textwidth]{{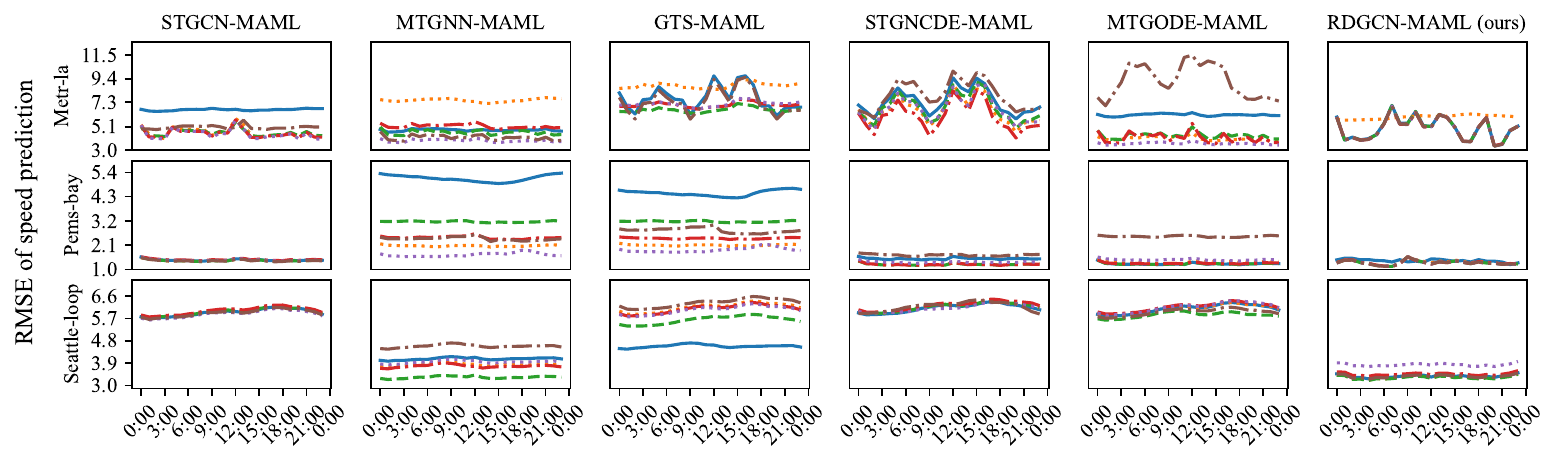}}}
    %\caption{fig1}
    % \vspace{-.08in}
    \hfill%

    \subfloat[Baseline models and RDGCN trained on more than half a year of weekdays.]
    {\label{fig:mismatch-mae-app}
    \includegraphics[width=0.99\textwidth, height=0.266\textwidth]{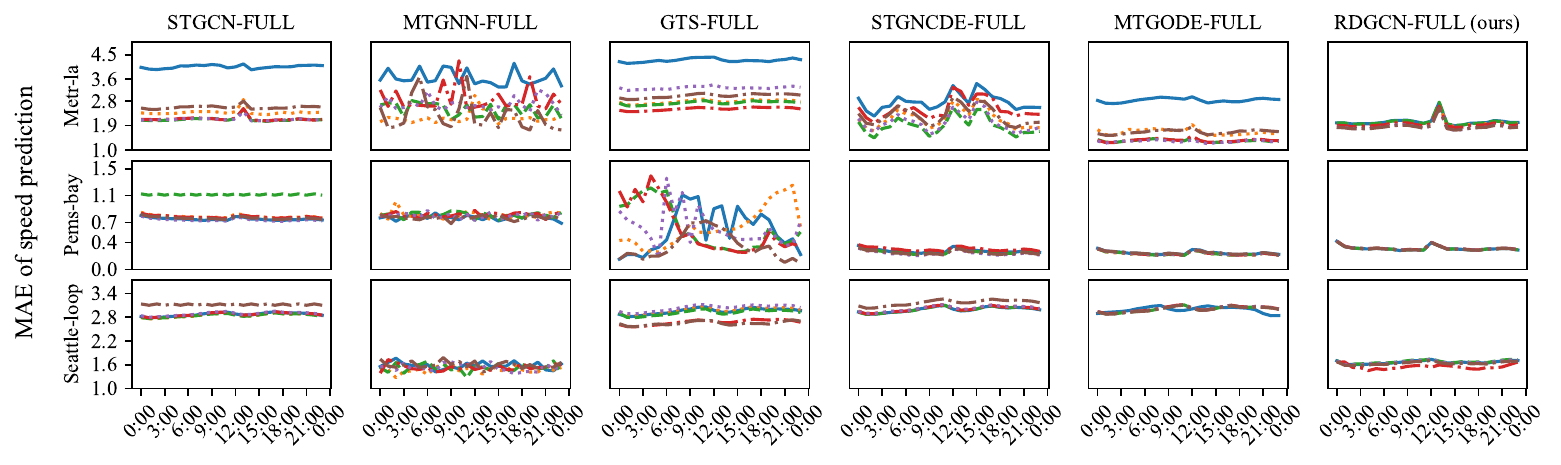}}
    
    \subfloat[Baseline models and RDGCN trained on more than half a year of weekdays.]
    {\label{fig:mismatch-maml-mse}
    % \vspace{-0.5in}
    \includegraphics[width=0.99\textwidth, height=0.266\textwidth]{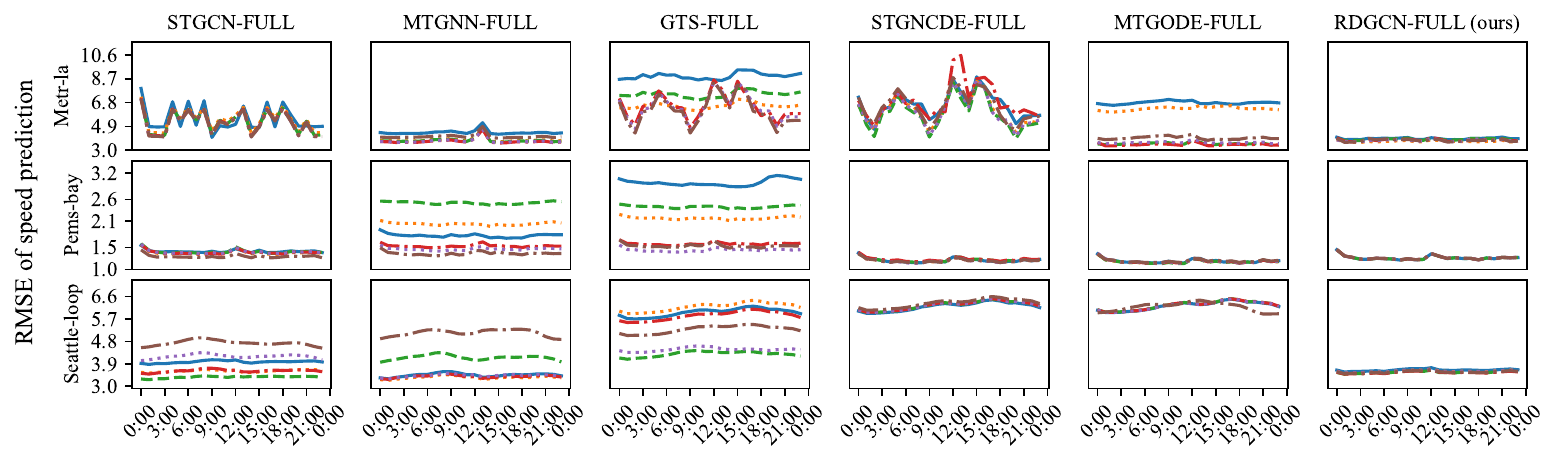}}

    % \subfloat[MSE: Limited and mismatched data with MAML augmentation.]
    % {\label{fig:mismatch-mse}
    % \hspace{-.0in}\includegraphics[width=0.99\textwidth, height=0.266\textwidth]{{figs/mismatch/mismatches.pdf}}}
    % %\caption{fig1}
    % \vspace{-.08in}
    % \hfill%
    % \subfloat[MSE: Sufficient mismatched data.]
    % {\label{fig:mismatch-maml-mse}
    % % \vspace{-0.5in}
    % \hspace{-.03in}\includegraphics[width=0.99\textwidth, height=0.266\textwidth]{figs/mismatch/mismatches-maml.pdf}}
    % \vspace{-.08in}
    \caption{
(a)(b) The results of RDGCN are very close regardless of the period of the training set. 
(c)(d) Even though all the models are trained using all available weekdays, the results of RDGCN are still closer, regardless of the period, compared to baseline models.}
\label{fig:mismatch-mse-all}
\end{figure*}

\tiny
\begin{table*}[htbp]
\caption{Numerical result of 
Figure~\ref{fig:mismatch-all}: the Mean and STD of prediction MAE and RMSE of RDGCN and baselines on three real-world datasets.
}

\centering
\resizebox{1.4\columnwidth}{!}{
\makebox[\textwidth][c]{
    \begin{tabular}{ccccccccccccc}
\toprule
             & \multicolumn{6}{c}{MAE} & \multicolumn{6}{c}{RMSE} \\
\cmidrule(lr){2-7}\cmidrule(lr){8-13}
With MAML & STGCN & MTGNN  & GTS & STGNCDE & MTGODE & RDGCN & STGCN & MTGNN  & GTS & STGNCDE & MTGODE & RDGCN \\ \midrule
Metr-la      & 2.47 $\pm$ 0.11 & 2.41 $\pm$ 0.22 & 2.55 $\pm$ 0.48 & 3.27 $\pm$ 0.47 & 2.82 $\pm$ 0.49  & \textbf{2.39} $\pm$ \textbf{0.08}  
& 5.28 $\pm$ 0.94 & 5.17 $\pm$ 1.16 & 7.55 $\pm$ 0.91 & 7.01 $\pm$ 1.28 & 5.41 $\pm$ 2.01 & 4.96 $\pm$ 0.83

\\
\midrule
Pems-bay       
& 1.03 $\pm$ 0.19 & 0.91 $\pm$ 0.21 & 0.96 $\pm$ 0.03 & 0.77 $\pm$ 0.06 & 0.86 $\pm$ 0.14  & \textbf{0.83} $\pm$ \textbf{0.03}   & 1.41 $\pm$ 0.05 & 2.86 $\pm$ 1.11 & 2.85 $\pm$ 0.84 & 1.44 $\pm$ 0.16 & 1.58 $\pm$ 0.44  & 1.40 $\pm$ 0.05 

\\
\midrule
Seattle-loop & 2.20 $\pm$ 0.08 & 2.23 $\pm$ 0.24 & 2.34 $\pm$ 0.15  & 3.20 $\pm$ 0.07 & 3.17 $\pm$ 0.05  & \textbf{2.16} $\pm$ \textbf{0.05}
& 5.94 $\pm$ 0.14 & 3.92 $\pm$ 0.37 & 5.80 $\pm$ 0.60 & 6.16 $\pm$ 0.17 & 6.04 $\pm$ 0.19 & 3.44 $\pm$ 0.18
\\
\midrule
% \cmidrule(lr){2-5}\cmidrule(lr){6-9}
FULL & STGCN & MTGNN  & GTS & STGNCDE & MTGODE & RDGCN & STGCN & MTGNN  & GTS & STGNCDE & MTGODE & RDGCN \\ \midrule

Metr-la      & 2.57 $\pm$ 0.68 & 3.11 $\pm$ 0.48 & 3.44 $\pm$ 0.47 & 2.77 $\pm$ 0.35 & 2.31 $\pm$ 0.43 & \textbf{2.38} $\pm$ \textbf{0.13}  
& 5.31 $\pm$ 0.92 & 4.02 $\pm$ 0.31 & 7.04 $\pm$ 1.20 & 6.43 $\pm$ 1.24 & 4.70 $\pm$ 1.38 & 3.90 $\pm$ 0.10
\\
\midrule
Pems-bay       
& 1.38 $\pm$ 0.06 & 1.85 $\pm$ 0.38 & 2.08 $\pm$ 0.51 & 0.83 $\pm$ 0.09 & 0.79 $\pm$ 0.02 & \textbf{0.74} $\pm$ \textbf{0.02}  
& 1.37 $\pm$ 0.06 & 1.85 $\pm$ 0.38 & 2.08 $\pm$ 0.53 & 1.38 $\pm$ 0.04 & 1.36 $\pm$ 0.04 & 1.38 $\pm$ 0.04 

\\
\midrule
Seattle-loop & 2.90 $\pm$ 0.10 & 2.81 $\pm$ 0.65 & 3.11 $\pm$ 0.11 & 3.32 $\pm$ 0.07 & 3.21 $\pm$ 0.05 & \textbf{2.18} $\pm$ \textbf{0.06}
& 3.91 $\pm$ 0.45 & 3.81 $\pm$ 0.65 & 5.33 $\pm$ 0.74 & 6.25 $\pm$ 0.17 & 6.22 $\pm$ 0.17 & 3.58 $\pm$ 0.05
\\

\bottomrule

\label{tab:final}
\end{tabular}%
}}
\end{table*}%
\normalsize

\textbf{Experimental results.}
The Mean and STD of prediction MAE (resp. RMSE) of each model with MAML augmentation and with full weekday training set are shown in Table~\ref{tab:final} (i.e., the Mean and STD of all points on each subfigure in Figure~\ref{fig:mismatch}, Figure~\ref{fig:mismatch-maml}, Figure~\ref{fig:mismatch-mse} and Figure~\ref{fig:mismatch-maml-mse}), respectively. 
Table~\ref{tab:final} shows that RDGCN has lower MAE (resp. RMSE) and lower variance compared with baselines under limited training set with MAML augmentation, and the gain of adding more data on RDGCN is limited, which are consistent with our observation in Figure~\ref{fig:mismatch-all} in Section~\ref{sec:result}.

\begin{figure}
% \vspace{-.2in}
\centering
\includegraphics[width=\columnwidth]{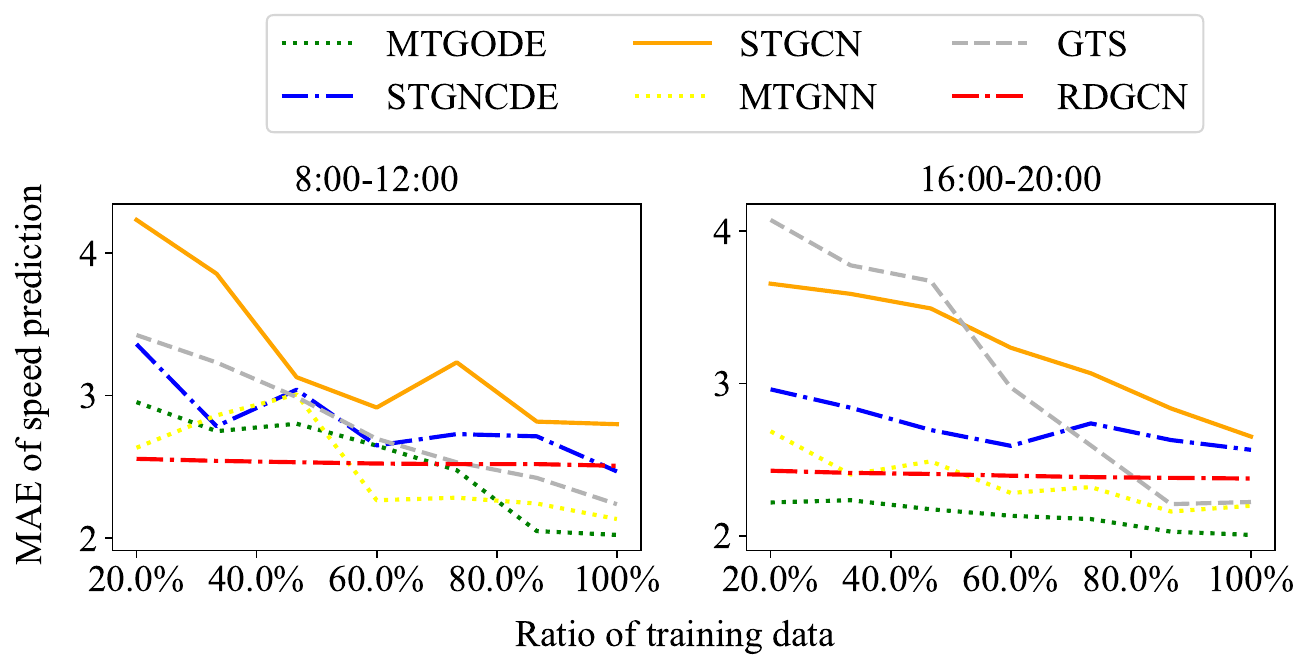} % Reduce the figure size so that it is slightly narrower than the column. Don't use precise values for figure width.This setup will avoid overfull boxes.
\caption{Feeding more training data does not lead to a significant change in the MAE of RDGCN's prediction.}
\label{fig:num_training-data}
\end{figure}

\textbf{Impact of data volume.}
We further investigate the influence of training data volume on the performance of baseline models and RDGCN under a mismatched setting. We focus on assessing the adequacy of training data for both morning rush hour (8:00-12:00) and evening rush hour (16:00-20:00) scenarios using the Metr-la dataset. These periods exhibit considerable patterns and exhibit relatively minor mismatches between training and test datasets. To this end, we randomly select contiguous weekdays, ranging from $20\%$ to the entire dataset, for training the models. The MAE of speed prediction across varying quantities of training data is shown in Figure~\ref{fig:num_training-data}.

Figure~\ref{fig:num_training-data} showcases the performance characteristics of RDGCN and baseline models over the specified time intervals. Remarkably, the performance of RDGCN remains consistent irrespective of the training dataset size. Conversely, the predictive capabilities of STGNCDE and MTGODE are notably contingent upon the amount of training data employed. The observed trend underscores increased training data volume directly correlates with enhanced prediction accuracy. In the morning rush hour, MTGODE achieves optimal performance with approximately $75\%$ of training data (equivalent to 60 weekdays), while STGNCDE demonstrates comparable performance when trained on the entire weekday dataset. We note that the superiority of RDGCN over baseline models is not universally consistent, as elucidated earlier. Notably, integrating domain differential equations drastically reduces the hypothesis class's size, thereby filtering out erroneous hypotheses often prevalent in conventional black-box graph learning models. Consequently, domain-differential-equation-informed GCNs exhibit remarkable robustness on relatively smaller training datasets.

\subsection{Analysis of SIRGCN in ILI Prediction}

\textbf{Do the infection rates vary among different vertices?}
In this section, we delve into the question of whether we require an individual infection rate for each vertex in ILI prediction. We specifically examine two approaches: one where we assign a unique infection rate, denoted as $\beta_i$, to each vertex $i$, resulting in a SIRGCN with $n$ infection rates (SIRGCN-n), and another approach where we assign a single infection rate, denoted as $\beta$, to all vertices (SIRGCN-1). 
We report the MAE and RMSE of the prediction under mismatched data (train using Winter-Summer data and test using Spring-Fall data) in Table~\ref{tab:sir-beta-metrices}.

Table~\ref{tab:sir-beta-metrices} shows that employing multiple infection rates leads to more accurate predictions, particularly in the case of the US-state dataset. 
By assigning individual infection rates to each vertex, we achieve a reduction of $2.4\%$ in MAE (and $1.6\%$ in RMSE). However, the advantage of utilizing multiple infection rates is less pronounced ($< 1\%$) in the ILI prediction of Japan. 
There could be two potential explanations for this phenomenon. First, the size of Japan's prefectures is not as substantial as that of the states in the United States. Second, the climate across Japan is relatively homogeneous, whereas the climate across different states of the United States exhibits significant variations, such as wet coastal areas and dry inland areas.

\begin{table}[htbp]
\small
  \caption{Evaluation models under mismatched data.}
\resizebox{\columnwidth}{!}{
% \makebox[\columnwidth][c]{
    \begin{tabular}{ccccc}
\toprule
             & \multicolumn{2}{c}{MAE} & \multicolumn{2}{c}{RMSE} \\
\cmidrule(lr){2-3}\cmidrule(lr){4-5}
 & SIRGCN-1 & SIRGCN-n  & SIRGCN-1 & SIRGCN-n \\ \midrule
 Japan-Prefectures & 344 $\pm$ 22 & 342 $\pm$ 22 & 871 $\pm$ 43  & 863 $\pm$ 44 \\
 US-States & 42 $\pm$ 4  & 41 $\pm$ 4 & 123 $\pm$ 10 & 121 $\pm$ 10 \\
\bottomrule
\end{tabular}%
}

\label{tab:sir-beta-metrices}
\end{table}%
\normalsize

\textbf{Predictions in Different Seasons.}
Learning patterns across different trends becomes challenging when baseline models are not trained using the same trend. For example, during Winter, the infectious number shows an increasing trend, whereas during Spring, it exhibits a decreasing trend. Figure~\ref{fig:covid-rapid} shows the predicted number of infectious cases alongside the ground truth data, revealing that SIRGCN's prediction aligns better with the ground truth. Conversely, EpiGNN's prediction performs poorly during the decline phase and when the number of infections approaches $0$.

In the case of US-State ILI prediction in May 2014, both COLAGNN and EPIGNN fail to make accurate predictions around the peak, while SIRGCN demonstrates its effectiveness during the corresponding period, with the help of SIR-network model.

\begin{figure}
\centering
\includegraphics[width=0.95\columnwidth]
{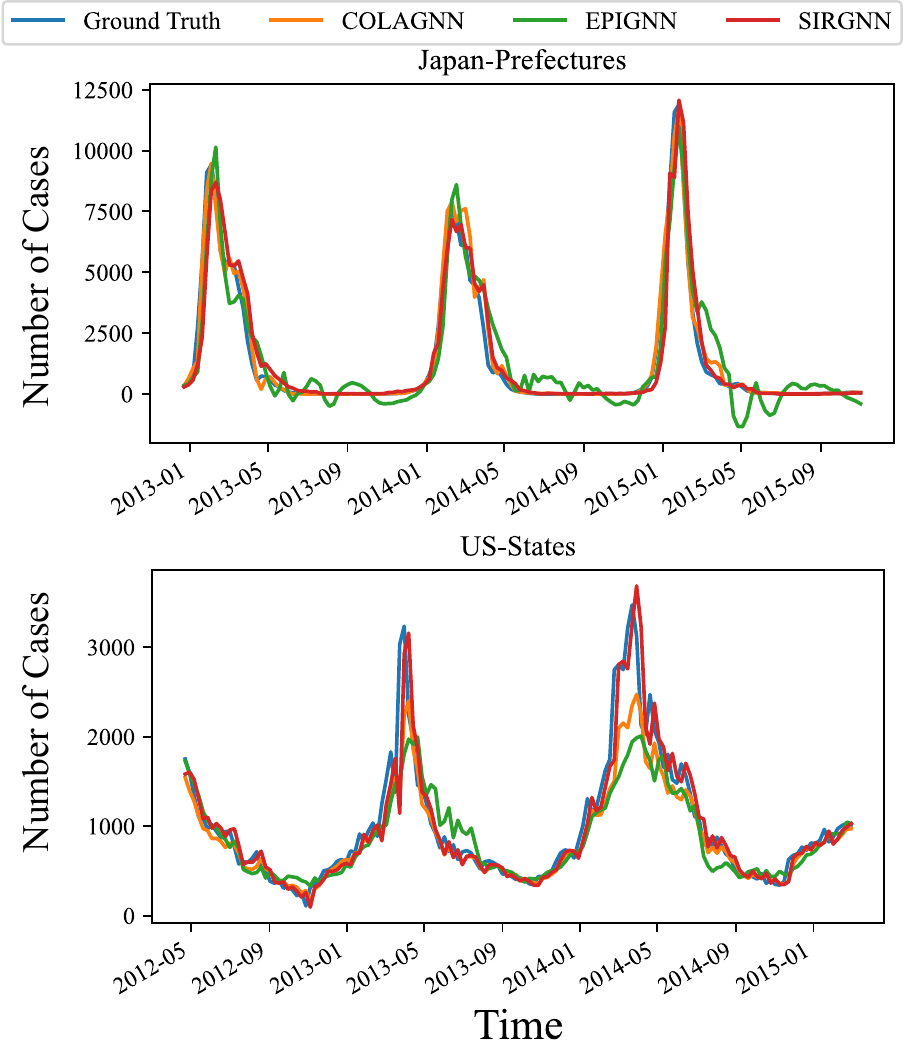} 
\caption{SIRGCN can make accurate predictions in the decreasing phase, while EpiGNN makes bad predictions in the corresponding phase.
}
\label{fig:covid-rapid}
\end{figure}

\section{Model Efficiency in Computation Time}\label{sec:efficiency}
The training time and inference time (on two NVIDIA-2080ti graphic cards) of STGCN, MTGNN, GTS, and
RDGCN on the Metr-la dataset are demonstrated in Table~\ref{table:efficiency}.
It's observed that RDGCN takes less time in both training and inference than the other models,
since the RDGCN contains significantly less number of parameters than the baseline models. 
% While the spatial convolutional layers of all the baseline models have identical complexity (linearly scales with number of vertices), STGCN, MTGNN and GTS have additional rich temporal layers with more than a few tens of thousand parameters. 
% \blue{The last sentence is attacked by AAAI reviewer. Do we need a table record the number of parameters?}
%In a traffic network, the number of edges is only slightly larger than the number of vertices, which implies the number of parameter of RDGCN ($2|\mathcal{V}|+2|\mathcal{E}|$) is close to the number of parameter of two GCN layers ($4 |\mathcal{V}|$) in STGCN, MTGNN and GTS, thus saving the parameters of the complex time convolutional layers.

\small
\begin{table}[!h]
\vspace{-.1in}
\small
\centering
\caption{The computation time on the Metr-la dataset. 
% Each model is trained on data from 4 hours of data in 12 days.
}

\resizebox{\columnwidth}{!}{

\begin{tabular}{l|cccc}
    \midrule
    & & \# Parameters&Training (s/epoch) & Inference (s) \\
    \midrule
    \multirow{6}{*}{Metr-la} & STGCN & 458865 & 0.5649 & 0.0232 \\
    & MTGNN & 405452 & 0.5621 & 0.0607 \\
    & GTS & 38377299 & 1.0632 & 0.1641 \\
    & STGNCDE & 374904 & 1.7114 & 0.3729 \\
    & MTGODE & 138636 & 1.6158 & 0.3491 \\
    & RDGCN & 872 & 0.0308 & 0.0037 \\
    \midrule
    \multirow{3}{*}{Japan-prefectures} & ColaGNN & 4272 & 0.0297 & 0.0065 \\
    & EpiGNN & 16875 & 0.0311 & 0.0073 \\
    & SIRGCN & 181 & 0.0289 & 0.0063 \\
    \midrule    
\end{tabular}}
% \vspace{-.1in}

\label{table:efficiency}
\end{table}
\normalsize

The training and inference time of ColaGNN, EpiGNN, and SIRGCN are shown in Table~\ref{table:efficiency}. 
SIRGCN has significantly less number of parameters than the baseline models.
We acknowledge that the computational time of SIRGCN is similar to that of the baseline models, as the baselines are not as deep or dense as traffic prediction models and do not require a large amount of data for training.

\small
\begin{table}[!h]
\vspace{-.1in}
\small
\centering
\caption{The computation time on the Japan-Prefectures dataset. 
% Each model is trained on data from 4 hours of data in 12 days.
}

% \resizebox{.8\columnwidth}{!}{

% \begin{tabular}{l|ccc}
%     \midrule
%     & \# Parameters&Training (s/epoch) & Inference (s) \\
%     \midrule
%     ColaGNN & 4272 & 0.0297 & 0.0065 \\
%     EpiGNN & 16875 & 0.0311 & 0.0073 \\
%     SIRGCN & 181 & 0.0289 & 0.0063 \\
%     \midrule
% \end{tabular}}
% \vspace{-.1in}

\label{table:efficiency-sir}
\end{table}
\normalsize

\section{Experimental Settings} \label{subsec:MAML}

\textbf{Evaluation.}  
We assume that all zeros in the datasets are missing values, and we remove the predicted speed
% of the sensors in the period
when the ground truth is $0$, or when the last speed recorded is $0$. 

\textbf{Hyperparameter Settings.} 
RDGCN and SIRGCN are optimized via Adam. The batch size is set as 64. The learning rate is set as $0.001$, and the early stopping strategy is used with a patience of $30$ epochs. 
In traffic speed prediction, the training and validation set are split by a ratio of 3:1 from the weekday subset, and the test data is sampled from the weekend subset with different patterns.
As for baselines,
we use identical hyperparameters as released in their works.
In ILI prediction, the training and validation set are split by a ratio of 5:2 from the Winter-Summer subset, and the test data is sampled from the Spring-Fall subset with different patterns.
The Susceptible population at the beginning of each ILI period is $10\%$ of the total population in each Prefectures or States. 
As for baselines,
we also use identical hyperparameters as released in their works.
We approximate the total number of populations by the average of the annual sum of infectious cases, multiplied by $10$.
% to train baselines models.

% \begin{figure*}[!h]
% \vspace{-.28in}
% \centering
% \includegraphics[width=0.99\textwidth]{figs/mismatch/mismatches-maml.pdf} % Reduce the figure size so that it is slightly narrower than the column. Don't use precise values for figure width.This setup will avoid overfull boxes.
% \caption{The results of RDGCN enhanced by MAML are very close regardless of the time period of the training set. 
% The figures show the averaged MAE over five repeated experiments, where models are trained with different sequences of 12 consecutive weekdays each but tested with the same weekend test set.
% % The value of each point is the average of the predictions given by the model, which is trained with the data in the same required time intervals but on five different weekday dataset. 
% }
% \label{fig:mismatch-maml}
% \vspace{-0.2in}
% \end{figure*}

% In this section, we compare the performance of our appproach and baseline models, trained with limited and mismatched data, and with the augmentation of MAML respectively. 
\textbf{MAML Settings.}
Our experiment involves the following steps: (1)We randomly select sequences of 12 consecutive weekdays (same as the Limited and Mismatched Data experiment.), and sample four-hour data as the training set. We evaluate the model with hourly data on weekends. (2) We divide the training set into two equal parts: the support set and the query set. (3) We use the support set to compute adapted parameters. (4) We use the adapted parameters to update the MAML parameters on the query set. (5) We repeat this process $200$ times to obtain initial parameters for the baseline model. (6) We train baselines using the obtained initial parameters. The learning rate for the inner loop is $0.00005$, and for the outer loop is $0.0005$, and MAML is trained for $200$ epochs. 

\end{document}